\documentclass[11pt]{article}
\usepackage[letterpaper, left=1in, right=1in, top=1in, bottom=1in]{geometry}

\usepackage{parskip}

\usepackage[dvipsnames]{xcolor}
\usepackage[colorlinks=true, linkcolor=blue, citecolor=blue, pagebackref, bookmarks=true]{hyperref}
\usepackage{bookmark}
\usepackage{microtype}

\usepackage{algorithm}
\usepackage{algpseudocode}

\usepackage{natbib}
\bibpunct{(}{)}{;}{a}{,}{,}

\usepackage{amsthm}
\usepackage{mathtools}
\usepackage{amsmath}
\usepackage{bbm}
\usepackage{amsfonts}
\usepackage{amssymb}

\usepackage{MnSymbol} 

\usepackage{xpatch}

\theoremstyle{definition}  

\newtheorem{lemma}{Lemma}

\newtheorem{proposition}{Proposition}

\newtheorem{assumption}{Assumption}

\newtheorem{remark}{Remark}
\theoremstyle{plain}
\newtheorem{example}{Example}
\newtheorem{theorem}{Theorem}
\newtheorem{definition}{Definition}

\xpatchcmd{\proof}{\itshape}{\normalfont\proofnameformat}{}{}
\newcommand{\proofnameformat}{\bfseries}

\usepackage{prettyref}
\newcommand{\pref}[1]{\prettyref{#1}}
\newcommand{\pfref}[1]{Proof of \prettyref{#1}}
\newcommand{\savehyperref}[2]{\texorpdfstring{\hyperref[#1]{#2}}{#2}}
\newrefformat{eq}{\savehyperref{#1}{\textup{(\ref*{#1})}}}
\newrefformat{eqn}{\savehyperref{#1}{Equation~\ref*{#1}}}
\newrefformat{lem}{\savehyperref{#1}{Lemma~\ref*{#1}}}
\newrefformat{def}{\savehyperref{#1}{Definition~\ref*{#1}}}
\newrefformat{line}{\savehyperref{#1}{line~\ref*{#1}}}
\newrefformat{thm}{\savehyperref{#1}{Theorem~\ref*{#1}}}
\newrefformat{corr}{\savehyperref{#1}{Corollary~\ref*{#1}}}
\newrefformat{sec}{\savehyperref{#1}{Section~\ref*{#1}}}
\newrefformat{app}{\savehyperref{#1}{Appendix~\ref*{#1}}}
\newrefformat{ass}{\savehyperref{#1}{Assumption~\ref*{#1}}}
\newrefformat{ex}{\savehyperref{#1}{Example~\ref*{#1}}}
\newrefformat{fig}{\savehyperref{#1}{Figure~\ref*{#1}}}
\newrefformat{alg}{\savehyperref{#1}{Algorithm~\ref*{#1}}}
\newrefformat{rem}{\savehyperref{#1}{Remark~\ref*{#1}}}
\newrefformat{conj}{\savehyperref{#1}{Conjecture~\ref*{#1}}}
\newrefformat{prop}{\savehyperref{#1}{Proposition~\ref*{#1}}}
\newrefformat{proto}{\savehyperref{#1}{Protocol~\ref*{#1}}}
\newrefformat{prob}{\savehyperref{#1}{Problem~\ref*{#1}}}
\newrefformat{claim}{\savehyperref{#1}{Claim~\ref*{#1}}}

\DeclarePairedDelimiter{\abs}{\lvert}{\rvert} %
\DeclarePairedDelimiter{\brk}{[}{]}
\DeclarePairedDelimiter{\crl}{\{}{\}}
\DeclarePairedDelimiter{\prn}{(}{)}
\DeclarePairedDelimiter{\nrm}{\|}{\|}
\DeclarePairedDelimiter{\tri}{\langle}{\rangle}
\DeclarePairedDelimiter{\dtri}{\llangle}{\rrangle}

\DeclarePairedDelimiter{\floor}{\lfloor}{\rfloor}

\let\Pr\undefined

\DeclareMathOperator*{\En}{\mathbb{E}}
\DeclareMathOperator{\Enn}{\mathbb{E}}

\DeclareMathOperator{\Pr}{Pr}

\DeclareMathOperator*{\argmin}{arg\,min} 
\DeclareMathOperator*{\argmax}{arg\,max}

\newcommand{\ls}{\ell}
\newcommand{\ind}{\mathbbm{1}}    
\newcommand{\pmo}{\crl*{\pm{}1}}
\newcommand{\eps}{\epsilon}
\newcommand{\veps}{\varepsilon}

\newcommand{\ldef}{\vcentcolon=}

\newcommand{\mc}[1]{\mathcal{#1}}
\newcommand{\wt}[1]{\widetilde{#1}}

\def\ddefloop#1{\ifx\ddefloop#1\else\ddef{#1}\expandafter\ddefloop\fi}
\def\ddef#1{\expandafter\def\csname bb#1\endcsname{\ensuremath{\mathbb{#1}}}}
\ddefloop ABCDEFGHIJKLMNOPQRSTUVWXYZ\ddefloop
\def\ddefloop#1{\ifx\ddefloop#1\else\ddef{#1}\expandafter\ddefloop\fi}
\def\ddef#1{\expandafter\def\csname b#1\endcsname{\ensuremath{\mathbf{#1}}}}
\ddefloop ABCDEFGHIJKLMNOPQRSTUVWXYZ\ddefloop
\def\ddef#1{\expandafter\def\csname c#1\endcsname{\ensuremath{\mathcal{#1}}}}
\ddefloop ABCDEFGHIJKLMNOPQRSTUVWXYZ\ddefloop
\def\ddef#1{\expandafter\def\csname h#1\endcsname{\ensuremath{\widehat{#1}}}}
\ddefloop ABCDEFGHIJKLMNOPQRSTUVWXYZ\ddefloop
\def\ddef#1{\expandafter\def\csname hc#1\endcsname{\ensuremath{\widehat{\mathcal{#1}}}}}
\ddefloop ABCDEFGHIJKLMNOPQRSTUVWXYZ\ddefloop
\def\ddef#1{\expandafter\def\csname t#1\endcsname{\ensuremath{\widetilde{#1}}}}
\ddefloop ABCDEFGHIJKLMNOPQRSTUVWXYZ\ddefloop
\def\ddef#1{\expandafter\def\csname tc#1\endcsname{\ensuremath{\widetilde{\mathcal{#1}}}}}
\ddefloop ABCDEFGHIJKLMNOPQRSTUVWXYZ\ddefloop

\newcommand{\vol}{\mathrm{Vol}}

\let\wt\undefined

\newcommand{\wl}{\text{WL}}
\newcommand{\logistic}{\text{Logistic}}
\newcommand{\predict}{\text{Predict}}
\newcommand{\update}{\text{Update}}

\newcommand{\wt}[1]{\widetilde{#1}}
\newcommand{\mb}[1]{\boldsymbol{#1}}
\newcommand{\R}{\bbR}	

\newcommand{\ones}{\mb{1}}

\newcommand{\clip}{\textrm{Clip}}

\newcommand{\had}{\odot}
\newcommand{\diag}{\textrm{diag}}

\newcommand{\hy}{\yh}

\newcommand{\sgn}{\textnormal{sgn}}

\newcommand{\yr}[1][n]{y_{1:#1}}

\newcommand{\y}{\mathbf{y}}

\newcommand{\grad}{\nabla}

\newcommand{\tens}{\otimes{}}

\newcommand{\yh}{\hat{y}}
\newcommand{\zh}{\hat{z}}

\newcommand{\zt}{\tilde{z}}

\newcommand{\x}{\mathbf{x}}

\usepackage{enumitem}

\newcommand{\smooth}[1][\mu]{\mathrm{smooth}_{#1}}
\renewcommand{\clip}[1][\delta]{\mathrm{clip}_{#1}}
\newcommand{\wdim}{D_{\cW}}

\newcommand{\F}{\cF}
\newcommand{\p}{\mathbf{p}}
\newcommand{\vv}{\mathbf{v}}
\newcommand{\uu}{\mathbf{u}}
\newcommand{\zz}{\mathbf{z}}
\newcommand{\ww}{\mathbf{w}}

\newcommand{\yhtree}{\mathbf{\hat{y}}}

\newcommand{\X}{\cX}

\newcommand{\logloss}{\ls_{\mathrm{log}}}

\newcommand{\lpred}{\hat{p}}
\newcommand{\phtree}{\mathbf{\hat{p}}}

\newcommand{\Fclip}{\cF^{\delta}}
\newcommand{\zo}{\crl*{0,1}}
\newcommand{\deltarange}{\brk*{\delta,1-\delta}}

\title{\textbf{Logistic Regression:\\ The Importance of Being Improper}}
\date{}

\author{
Dylan J. Foster\thanks{Cornell University} \quad Satyen Kale\thanks{Google Research} \quad Haipeng Luo\thanks{University of Southern California} \quad Mehryar Mohri\footnotemark[2]~\thanks{New York University} \quad Karthik Sridharan\footnotemark[1]
}

\begin{document}

\maketitle

\begin{abstract}

 Learning linear predictors with the logistic loss---both in stochastic and online settings---is a fundamental task in machine learning and statistics, with direct connections to classification and boosting. Existing ``fast rates'' for this setting exhibit exponential dependence on the predictor norm, and \cite{hazan2014logistic} showed that this is unfortunately unimprovable. Starting with the simple observation that the logistic loss is $1$-mixable, we design a new efficient \emph{improper} learning algorithm for online logistic regression that circumvents the aforementioned lower bound with a regret bound exhibiting a \emph{doubly-exponential} improvement in dependence on the predictor norm. This provides a positive resolution to a variant of the COLT 2012 open problem of \citet{mcmahan2012open} when improper learning is allowed. This improvement is obtained both in the online setting and, with some extra work, in the batch statistical setting with high probability. We also show that the improved dependence on predictor norm is near-optimal. 
 
 Leveraging this improved dependency on the predictor norm yields the following applications: (a) we give algorithms for online bandit multiclass learning with the logistic loss with an $\tilde{O}(\sqrt{n})$ relative mistake bound across essentially all parameter ranges, thus providing a solution to the COLT 2009 open problem of \citet{abernethyR09a}, and (b) we give an adaptive algorithm for online multiclass boosting with optimal sample complexity, thus partially resolving an open problem of \citet{beygelzimer2015optimal} and \citet{jung2017onlinemulticlass}. Finally, we give information-theoretic bounds on the optimal rates for improper logistic regression with general function classes, thereby characterizing the extent to which our improvement for linear classes extends to other parametric and even nonparametric settings.

\end{abstract}

\section{Introduction}
\label{sec:introduction}

Logistic regression is a classical model in statistics used for
estimating conditional probabilities \citep{Berkson1944}. The model,
also known as \emph{conditional maximum entropy model}
\citep{BergerDellaPietraDellaPietra1996}, has been extensively studied
in statistical and online learning and has been widely used in
practice both for binary classification and multi-class
classification in a variety of applications.

This paper presents a new study of logistic regression in online
learning.  The basic logistic regression problem consists of learning a
linear predictor with performance measured by the \emph{logistic loss}. In the online setting, when
the hypothesis class is that of $d$-dimensional linear predictors with
$\ls_{2}$ norm bounded by $B$, there are two main algorithmic
approaches to logistic regression: Online Gradient Descent \citep{Zinkevich03,shalev2007convex,nemirovski2009robust}, which
admits a regret guarantee of $O(B\sqrt{n})$ over $n$ rounds, and
Online Newton Step \citep{hazan2007logarithmic}, whose regret bound is in $O(de^{B}\log(n))$. While
the latter bound is logarithmic in $n$, its poor dependence on
$B$ makes it weaker and guarantees an improvement only when $B \ll \frac{1}{2}\log(n)$. The question of whether this dependence on $B$ could be improved was posed as an open problem in COLT 2012 by \citet{mcmahan2012open}. \citet{hazan2014logistic} answered this in the negative, showing a lower bound of $\Omega(\sqrt{n})$ for $B\geq{}\Omega(\log(n))$.

The starting point for this work is a simple observation: the logistic loss, when viewed as a function of the prediction and the true outcome, is $1$-mixable (see \pref{sec:prelims} for definitions). This observation can be used in conjunction with Vovk's Aggregating Algorithm~\citep{vovk1995game}, which leverages mixability in order to achieve regret bounds scaling logarithmically in an appropriate notion of complexity of the space of predictors, and can be implemented in \emph{polynomial time} in relevant parameters using MCMC methods (\pref{sec:logistic}). Mixability and efficient implementability open the door to fast rates for online logistic regression and related problems via \emph{improper learning}: using predictions that may not be linear in the instances $x_{t}$s.

The power of improper learning manifests itself in solutions we present for three open problems. First, we give an \emph{efficient} online learning algorithm that circumvents the lower bound of \citet{hazan2014logistic} via improper learning and attains a substantially more favorable regret guarantee of $O(d\log(Bn))$; this is a \emph{doubly-exponential improvement} of the dependence on the scale parameter $B$. This algorithm provides a positive resolution to to a variant of the open problem of \citet{mcmahan2012open} where improper predictions are allowed. Second, the same technique provides an algorithm (\pref{sec:bandit_multiclass}) for the \emph{online multiclass learning with bandit feedback problem} \citep{kakade2008efficient} with an $\tilde{O}(\sqrt{n})$ relative mistake bound with respect to the multiclass logistic loss. This algorithm provides a solution to an open problem of \citet{abernethyR09a}, improving upon the previous algorithm of \citet{hazan2011newtron} by providing the $\tilde{O}(\sqrt{n})$ mistake bound guarantee for all possible ranges of parameter sets. Third, the technique provides a new \emph{online multiclass boosting} algorithm (\pref{sec:online_boosting}) with optimal sample complexity, thus partially resolving an open problem from \citep{beygelzimer2015optimal,jung2017onlinemulticlass} (the algorithm is sub-optimal in the number of weak learners it uses, though it is no worse in this regard than previous adaptive algorithms). For clarity of exposition, descriptions of all of these applications are given as concisely as possible without presenting the results in the most general form possible.

We further present a series of new results for batch statistical
learning. We show how to convert our online improper logistic
regression algorithm into a solution admitting a high-probability
excess risk guarantee of $O(d\log(Bn)/n)$
(\pref{sec:online_to_batch}). While it is straightforward to achieve such a result in expectation using standard
online-to-batch conversion techniques, the a high-probability bound is more technically challenging. We achieve this using a new technique based on a
modified version of the ``boosting the confidence'' scheme proposed by
\citet{mehta2016fast} for exp-concave losses.  We also prove a lower bound showing that the logarithmic dependence on $B$ of the guarantee of our new algorithm cannot be improved.  Finally, we show how to (non-constructively)
generalize the $\log(B)$ dependence on predictor norm from linear to arbitrary function classes via sequential symmetrization and chaining arguments
(\pref{sec:general_class}). Our general bound indicates that the
extent to which dependence on the predictor range $B$ can be improved for general classes is completely determined by their (sequential) metric entropy. We also show how to extend this technique to the log loss, where we obtain a minimax rate for general function classes that uniformly improves on the minimax log loss rates in \cite{RakSri15}.

\subsection{Preliminaries}
\label{sec:prelims}

\paragraph{Notation.} Let $\R^d$ be the $d$-dimensional Euclidean space with $\langle \cdot, \cdot \rangle$ denoting the standard inner product in $\R^d$. Let $\|\cdot\|$ be a norm on $\R^d$ with dual norm denoted by $\|\cdot\|_\star$. In the multiclass learning problem, the input feature space is the set $\cX = \{x \in \R^d|\ \|x\|_\star \leq R\}$ for some unknown $R>0$. The number of output classes is $K$ and the set of output classes is denoted by $\brk{K} := \{1, 2, \ldots, K\}$. The set of distributions over $\brk{K}$ is denoted $\Delta_K$. Linear predictors are parameterized by weight matrices in $\R^{K \times d}$ so that for an input vector $x \in \cX$, $Wx \in \R^K$ is the vector of scores assigned by $W$ to the classes in $\brk{K}$. For a weight matrix $W$ and $k \in \brk{K}$, we denote by $W_k$ the $k$-th row of $W$. The space of parameter weight matrices is a convex set $\cW \subseteq \{W \in \R^{K \times D}|\ \forall k \in \brk{K}, \|W_k\| \leq B\}$ for some known parameter $B > 0$. Thus for all $x \in \cX$ and $W \in \cW$, we have $\|Wx\|_\infty \leq BR$.

 Define the softmax function $\mb{\sigma}:\bbR^{K}\to\Delta_{K}$ via $\mb{\sigma}(z)_{k} = \frac{e^{z_k}}{\sum_{j\in\brk{K}}e^{z_j}}$ for $k \in \brk{K}$. We also define a pseudoinverse for $\mb{\sigma}$ via $\mb{\sigma}^{+}(p)_{k}=\log(p_k)$ which has the property that for all $p \in \Delta_{K}$, we have $\mb{\sigma}(\mb{\sigma}^{+}(p)) = p$ and $\sum_{k \in \brk{K}}e^{\mb{\sigma}^{+}(p)_k} = 1$. The multiclass logistic loss, also referred to as \emph{softmax-cross-entropy} loss, is defined as $\ls: \bbR^{K}\times{}\brk{K}\to\bbR$ as $\ls(z, y) := -\log(\mb{\sigma}(z)_{y})$. 

 It will be convenient to overload notation and define a weighted version of the multiclass logistic loss function as follows: let $\cY\ldef\crl*{y\in\bbR^{K}_{+}\mid{} \nrm*{y}_{1}\leq{}L}$ for some known parameter $L > 0$. Then the weighted multiclass logistic loss function $\ls:\bbR^{K}\times{}\cY\to\bbR$ is defined by $\ls(z, y) = -\sum_{k\in\brk{K}}y_{k}\log(\mb{\sigma}(z)_{k})$. It can also be seen by straightforward manipulation that the above definition is equivalent to $\ls(z, y) = \sum_{j\in\brk{K}}y_j\log\prn*{1 + \sum_{k\neq{}j}e^{z_k-z_j}}$.
 
In the binary classification setting, the standard definition of the logistic loss function is (superficially) different: the label set is is $\{-1, 1\}$, and the logistic loss $\ls: \R \times \{-1, 1\} \rightarrow \R$ is defined as $\ls_\text{bin}(z, y) = \log(1 + \exp(-yz))$. Linear predictors are parameterized by weight vectors $w \in \R^d$ with $\|w\|_2 \leq B$, and the loss for a predictor with parameter $w \in \R^d$ on an example $(x, y) \in \R^d \times \{-1, 1\}$ is $\ls_\text{bin}(\langle w, x\rangle, y)$. This loss can be equivalently viewed in the multiclass framework above setting $K = 2$, $\cW = \{W \in \R^{2 \times d}|\ \nrm*{W_1}_2 \leq B, W_2 = 0\}$, and mapping the labels $1 \mapsto 1$ and $-1 \mapsto 2$.

Finally, we make frequent use of a smoothing operator $\smooth: \Delta_K \rightarrow \Delta_K$ for a parameter $\mu \in [0, 1/2]$, defined via $\smooth(p) = (1-\mu)p + \mu\ones{}/K$ where $\ones \in \R^K$ is the all ones vector. We use the notation $\mb{1}[\cdot]$ to denote the indicator random variable for an event.

\paragraph{Online multiclass logistic regression.} We use the following multiclass logistic regression protocol. Learning proceeds over a series of rounds indexed by $t=1,\ldots,n$. In each round $t$, nature provides $x_{t}\in\mathcal{X}$, and the learner selects prediction $\zh_{t}\in\bbR^{K}$ in response. Then nature provides an outcome $y_t \in \brk{K}$ or $y_t\in\cY$, depending on application, and the learner incurs multiclass logistic loss $\ls(\zh_t, y_t)$. The regret of the learner is defined to be $\sum_{t=1}^n \ell(\zh_t, y_t) - \inf_{W \in \cW} \sum_{t=1}^n \ell(Wx_t, y_t)$.

The learner is said to be \emph{proper} if it generates $\zh_t$ by choosing a weight matrix $W_t \in \cW$ \emph{before} observing the pair $(x_t, y_t)$ and setting $\zh_t = W_tx_t$. This is the standard protocol when the problem is viewed as an instance of online convex optimization, and is the setting for previous investigations into fast rates for logistic regression \citep{bach2010self,mcmahan2012open,bach2013non, bach2014adaptivity}, including the negative result of \citet{hazan2014logistic}. The more general online learning setting that is described above allows \emph{improper} learners which may generate $\zh_t$ arbitrarily using knowledge of $x_t$. 

\paragraph{Fast rates and mixability.} Conditions under which \emph{fast rates} for online/statistical learning (meaning that average regret or generalization error scales as $\tilde{O}(1/n)$ rather than $O(1/\sqrt{n})$) are achievable have been studied extensively (see \citep{van2015fast} and the references therein). For the purpose of this paper, a rather general condition on the structure of the problem that leads to fast rates is Vovk's notion of \emph{mixability}~\citep{vovk1995game}, which we define in an abstract setting below. Consider a prediction problem where the set of outcomes is $\cY$ and the set of predictions is $\cZ$, and the loss of a prediction on an outcome is given by a function $\ell: \cZ \times \cY \rightarrow \R$. For a parameter $\eta > 0$, the loss function $\ell$ is said to be $\eta$-mixable if for any probability distribution $\pi$ over $\cZ$, there exists a \emph{``mixed''} prediction $z_\pi \in \cZ$ such that for all possible outcomes $y \in \cY$, we have $\En_{z \sim \pi}[\exp(-\eta \ell(z, y))] \leq \exp(-\eta \ell(z_\text{mix}, y))$.

Now suppose that we are given a finite reference class of predictors $\cF$ consisting of functions $f: \cX \rightarrow \cZ$, where $\cX$ is the input space. The problem of online learning over $\cF$ with an $\eta$-mixable loss function admits an \emph{improper} algorithm, viz. Vovk's Aggregating Algorithm~\citep{vovk1995game}, with regret bounded by $\frac{\log|\cF|}{\eta}$, a \emph{constant} independent of the number of prediction rounds $n$. The algorithm simply runs the standard exponential weights/Hedge algorithm \citep{PLG} with learning rate set to $\eta$. In each round $t$, given an input $x_t$, the distribution over $\cF$ generated by the exponential weights algorithm induces a distribution over $\cZ$ via the outputs of the predictors on $x_t$, and the Aggregating Algorithm plays the mixed prediction for this distribution over $\cZ$. Finally, if $\cF$ is infinite, under appropriate conditions on $\cF$ fast rates can be obtained by running a continuous version of the same algorithm. This is the strategy we employ in this paper for the logistic loss.


\section{Improved Rates for Online Logistic Regression}
\label{sec:logistic}

We start by providing a simple proof of the mixability of the multiclass logisitic loss function for the case when the outcomes $y$ is a class in $\brk{K}$ (i.e. the unweighted case).
\begin{proposition} \label{prop:unweighted-mixability}
The unweighted multiclass logistic loss $\ls: \R^K \times \brk{K} \rightarrow \R$ defined as $\ls(z, y) = -\log(\mb{\sigma}(z)_y)$ is $1$-mixable.
\end{proposition}
\begin{proof}
The proof is by construction. Given a distribution $\pi$ on $\R^K$, define $z_\pi = \mb{\sigma}^{+}(\En_{z \sim \pi}[\mb{\sigma}(z)])$. Now, for any $y \in \brk{K}$, we have $\En_{z \sim \pi}[\exp(-\ls(z, y))] = \En_{z \sim \pi}[\mb{\sigma}(z)_y] = \mb{\sigma}(z_\pi)_y = \exp(-\ls(z_\pi, y))$. The second equality above uses the fact that for any $p \in \Delta_K$, $\mb{\sigma}(\mb{\sigma}^{+}(p)) = p$. Thus, $\ls$ is $1$-mixable.
\end{proof}
With a little more work, we can prove that the weighted multiclass logistic loss function is also mixable with a constant that inversely depends on the total weight. The proof appears in \pref{app:proofs}.
\begin{proposition} \label{prop:generalized_multiclass_log_mixable}
Let $\cY\ldef\crl*{y\in\bbR^{K}_{+}\mid{} \nrm*{y}_{1}\leq{}L}$ for some parameter $L > 0$. The weighted multiclass logistic loss $\ls: \R^K \times \cY \rightarrow \R$ defined as $\ls(z, y) = -\sum_{k\in\brk{K}}y_{k}\log(\mb{\sigma}(z)_{k})$ is $\frac{1}{L}$-mixable. For any distribution $\pi$ on $\R^K$, the mixed prediction $z_\pi = \mb{\sigma}^{+}(\En_{z \sim \pi}[\mb{\sigma}(z)])$ certifies $\frac{1}{L}$-mixability of $\ls$.
\end{proposition}

We are now ready to state a variant of Vovk's Aggregating Algorithm, \pref{alg:mixing_multiclass} for the online multiclass logistic regression problem from \pref{sec:prelims}, operating over a class of linear predictors parameterized by weight matrices $W$ in some convex set $\cW$. The algorithm and its regret bound (proved in \pref{app:proofs}) are given in some generality that is useful for applications.

\begin{algorithm}[h]
\caption{}
\label{alg:mixing_multiclass}
\begin{algorithmic}[1]
\Procedure{}{decision set $\cW$, smoothing parameter $\mu\in\brk{0,1/2}$.}
\State Initialize $P_1$ to be the uniform distribution over $\cW$.
\For{$t=1,\ldots,n$}
\State{Obtain $x_t$ and predict $\zh_{t}=\mb{\sigma}^{+}\prn*{\smooth\prn*{\En_{W\sim{}P_t}\brk*{\mb{\sigma}(Wx_t)}}}$.}
\State Obtain $y_t$ and define $P_{t+1}$ as the distribution over $\cW$ with density \hspace*{1in} $P_{t+1}(W) \propto \exp\prn{-\tfrac{1}{L}\textstyle{\sum}_{s=1}^{t}\ls(Wx_s, y_s)}$.
\EndFor
\EndProcedure
\end{algorithmic}
\end{algorithm}

\begin{theorem}
\label{thm:multiclass_logistic_regret}
The regret of \pref{alg:mixing_multiclass} is bounded by
\begin{equation}
\label{eq:regret_main}
\sum_{t=1}^{n}\ls(\zh_t,y_t) - \inf_{W\in\cW}\sum_{t=1}^{n}\ls(Wx_t,y_t) \leq{} 5LD_{\cW}\cdot{}\log\prn*{\frac{BRn}{D_{\cW}} + e} + 2\mu\sum_{t=1}^{n}\nrm*{y_t}_{1},
\end{equation}
where $D_{\cW}\ldef{}\mathrm{dim}(\cW)\leq{}dK$ is the linear-algebraic dimension of $\cW$.
The predictions $(\zh_t)_{t\leq{}n}$ generated by the algorithm satisfy $\nrm*{\zh_t}_{\infty} \leq{} \log(K/\mu)$.
\end{theorem}
Increasing the smoothing parameter $\mu$ only degrades the performance of \pref{alg:mixing_multiclass}. However, smoothing ensures that each prediction $\zh_t$ is bounded, which is important for our applications.

For the special case of multiclass prediction when $y \in [K]$, this algorithm enjoys a regret bound of $O(dK\log(\frac{BRn}{dK}+e))$. It thus provides a positive resolution to the open problem of \citet{mcmahan2012open} (in fact, with an exponentially better dependence on $B$ than what the open problem asked for), using improper predictions to circumvent the lower bound of \citet{hazan2014logistic}.

Turning to efficient implementation, it has been noted (e.g. \citep{hazan2007logarithmic}) that log-concave sampling or integration techniques \citep{lovasz2006fast, lovasz2007geometry} can be applied to compute the expectation in \pref{alg:mixing_multiclass} in polynomial time. The following proposition makes this idea rigorous\footnote{A subtlety is that since $\hat{z}_t$ is evaluated inside the nonlinear logistic loss we cannot exploit linearity of expectation.} and is proven formally in \pref{app:efficient}. We note that this is not a practical algorithm, however, and obtaining a truly practical algorithm with a modest polynomial dependence on the dimension is a significant open problem.
\begin{proposition}
\label{prop:alg_polytime}
\pref{alg:mixing_multiclass} can be implemented approximately so that the regret bound \pref{eq:regret_main} is obtained up to additive constants in time $\mathrm{poly}(d, n, B, R, K, L)$.
\end{proposition}

Finally, to conclude this section we state a lower bound, which shows that the $\log(B)$ factor in the regret bound in \pref{thm:multiclass_logistic_regret} cannot be improved for most values of $B$. This lower bound is by reduction to learning halfspaces with a margin in a Perceptron-type setting: We first show that \pref{alg:mixing_multiclass} can be configured to give a mistake bound of $O\prn*{d\log(\log(n)/\gamma)}$ for binary classification with halfspaces and margin $\gamma$,\footnote{It is a folklore result that this type of margin bound can be obtained by running a variant of the ellipsoid method online.} then give a lower bound against this type of rate. 

For simplicity, the lower bound is only stated in the binary outcome settting and we use the standard definition of the binary logistic loss, $\ls_\text{bin}$ from \pref{sec:prelims}. The proof is in \pref{app:proofs}.
\begin{theorem}[Lower bound]
\label{thm:logb_lower_bound}
Consider the binary logistic regression problem over the class of linear predictors with parameter set $\cW = \{w \in \R^{d}|\ \nrm*{w}_2 \leq B\}$ with $B=\Omega(\sqrt{d}\log(n))$. Then for any algorithm for prediction with the binary logistic loss, there is a sequence of examples $(x_t, y_t) \in \R^d \times \{-1, 1\}$ for $t \in [n]$ with $\nrm*{x_t}_2\leq{}1$ such that the regret of the algorithm is $\Omega\prn*{
d\log\prn*{
\frac{B}{\sqrt{d}\log(n)}
}
}$.
\end{theorem}

\paragraph{Relation to Bayesian Model Averaging}

To the best of our knowledge, the mixability of the logistic loss has surprisingly not appeared in the literature. However, \pref{alg:mixing_multiclass} can be seen as an instance of Bayesian model averaging, and consequently the analysis of \cite{kakade2005online} can be applied to derive the same $O(d\log(Bn/d))$ regret bound as in \pref{thm:multiclass_logistic_regret} in the binary setting. Specifically, it suffices to apply their Theorem 2.2 with parameter $\nu^{2}=B^{2}/d$. This highlights that Bayesian approaches can have great utility even when analyzed outside of the Bayesian framework.


\section{Application: Bandit Multiclass Learning}
\label{sec:bandit_multiclass}
The now apply our techniques to the \emph{bandit multiclass} problem. This problem, first studied by \citet{kakade2008efficient}, considers the protocol of online multiclass learning in \pref{sec:prelims} with nature choosing $y_t \in [K]$ in each round, but with the added twist of bandit feedback: in each round, the learner predicts a class $\yh_t \sim p_t$ and receives feedback only on whether the prediction was correct or not, i.e. $\ind[\hy_t \neq y_t]$. The goal is to minimize regret with respect to a reference class of linear predictors, using some appropriate surrogate loss function for the 0-1 loss. 

\citet{kakade2009complexity} used the multiclass hinge loss $\ell_\text{hinge}(W, (x_t, y_t)) = \max_{k \in [K] \setminus \{y_t\}}[1 + \tri*{W_k, x_t} - \tri*{W_{y_t}, x_t}]_+$ and gave an algorithm based on the multiclass Perceptron algorithm achieving $O(n^{2/3})$ regret. For a Lipschitz continuous surrogate loss function, running the EXP4 algorithm \citep{auer2002nonstochastic} on a suitable discretization of the space of all linear predictors obtains $\tilde{O}(\sqrt{n})$ regret, albeit very inefficiently, i.e. with exponential dependence on the dimension. 
In COLT 2009, \citet{abernethyR09a} posed the open problem of obtaining an {\em efficient} algorithm for the problem with $O(\sqrt{n})$ regret. Specifically, they suggested the multiclass logistic loss as an appropriate surrogate loss function for the problem. \citet{hazan2011newtron} solved the open problem and obtained an algorithm, Newtron,  based on the Online Newton Step algorithm \citep{hazan2007logarithmic} with $\tilde{O}(\sqrt{n})$ regret for the case when norm of the linear predictors scales at most logarithmically in $n$. \citet{beygelzimerOZ17} also solved the open problem presenting a different algorithm called SOBA. SOBA is analyzed using a different family of surrogate loss functions parameterized by a scalar $\eta \in [0, 1]$ with $\eta = 0$ corresponding to the hinge loss and $\eta = 1$ corresponding to the squared hinge loss. For all values of $\eta \in [0, 1]$, SOBA simultaneously obtains relative bound mistake bounds of $\tilde{O}(\frac{1}{\eta}\sqrt{n})$ with the comparator's loss measured with respect to the corresponding loss function. 

Now we present an algorithm, OBAMA (for {\em Online Bandit Aggregation Multiclass Algorithm}), depicted in \pref{alg:bandit_multiclass} in \pref{app:bandit_multiclass_proofs}, that obtains an $\tilde{O}(\sqrt{n})$ relative mistake bound for the multiclass logistic loss, thus providing another solution to the open problem of \citet{abernethyR09a}. The mistake bound of OBAMA trumps that of Newtron, since both algorithms rely on the same loss function, and OBAMA obtains an $\tilde{O}(\sqrt{n})$ relative mistake bound on a larger range of parameter values compared to Newtron. While SOBA also has an $\tilde{O}(\sqrt{n})$ relative mistake bound, the two bounds are incomparable since they are relative to the comparator's loss measured using different loss functions.

\begin{theorem}
\label{thm:bandit_multiclass}
There is a setting of the smoothing parameter $\mu$ such that OBAMA enjoys the following mistake bound:
\[ 
\sum_{t=1}^{n}\ind\brk*{\yh_t\neq{}y_t}\leq\inf_{W \in \cW} \sum_{t=1}^n \ell(Wx_t, y_t) + O\left(\min\left\{dK^2e^{2BR}\log\prn*{\tfrac{BRn}{dK} + e},\ \sqrt{dK^2\log(\tfrac{BRn}{dK} + e)n}\right\}\right).
\]
\end{theorem}
This bound significantly improves upon that of Newtron~\citep{hazan2011newtron},
which is of order $O(dK^3\min\{\exp(BR)\log(n), BRn^{\frac{2}{3}}\})$ under the same setting and surrogate loss. The proof of \pref{thm:bandit_multiclass} appears in \pref{app:bandit_multiclass_proofs}.


\section{Application: Online Multiclass Boosting}
\label{sec:online_boosting}

Another application of our techniques is to derive adaptive online boosting algorithms with optimal sample complexity,
which improves the AdaBoost.OL algorithm of~\citet{beygelzimer2015optimal} for the binary classification setting
as well as its multiclass extension AdaBoost.OLM of~\citet{jung2017onlinemulticlass}.
We state our improved online boosting algorithm in the multiclass setting for maximum generality,
following the exposition and notation of~\citet{jung2017onlinemulticlass} fairly closely.

We consider the following online multiclass prediction setting with 0-1 loss.
In each round $t$, $t=1,\ldots,n$, the learner receives an instance $x_t\in\cX$, then selects a class $\yh_t \in \brk{K}$ ,
and finally observes the true class $y_t \in\brk{K}$. The goal is to minimize the total number of mistakes
$
\sum_{t=1}^{n}\ind\crl*{\yh_t\neq{}y_t}.
$

In the boosting setup, we are interested in obtaining strong mistake bounds with the help of \emph{weak learners}.
Specifically, the learner is given access to $N$ copies of a weak learning algorithm for a cost-sensitive classification task.
Each weak learner $i\in\brk{N}$ works in the following protocol:
for time $t=1,\ldots,n$, 1) receive $x_t\in\cX$ and cost matrix $C_{t}^{i} \in \cC$;
2) predict class $l_t^{i}\in\brk{K}$;
3) receive true class $y_t\in\brk{K}$ and suffer loss $C_{t}^i(y_t, l_t^i)$. Here $\cC$ is some fixed cost matrices class and we follow~\citep{jung2017onlinemulticlass} to restrict to $\mc{C} = \crl*{C\in\bbR_+^{K\times{}K}\mid{} \forall{}y \in [K], C(y,y) = 0 \text{ and }  \nrm*{C(y, \cdot)}_{1}\leq{}1\;}$.

To state the weak learning condition, we define a randomized baseline $u_{\gamma,y}\in\Delta_{K}$ 
for some edge parameter $\gamma \in [0,1]$ and some class $y \in\brk{K}$,
so that $u_{\gamma,y}(k) = (1-\gamma)/K$ for $k\neq{}y$
and $u_{\gamma,y}(k) = (1-\gamma)/K + \gamma$ for $k=y$.
In other words, $u_{\gamma,y}$ puts equal weight to all classes except for the class $y$ which gets $\gamma$ more weight.
The assumption we impose on the weak learners is then that their performance is comparable to that of 
a baseline which always picks the true class with slightly higher probability than the others, formally stated below.
\begin{definition}[Weak Learning Condition~\citep{jung2017onlinemulticlass}]
\label{def:WLC}
An environment and a learner outputting $(l_t)_{t\leq{}n}$ satisfy the multiclass weak learning condition 
with edge $\gamma$ and sample complexity $S$ if for all outcomes $(y_t)_{t\leq{}n}$ and cost matrices $(C_{t})_{t\leq{}n}$ from the set $\cC$
adaptively chosen by the environment, we have%
\footnote{This is in fact a weaker weak learning condition than that of~\citep{jung2017onlinemulticlass}, which also allows weights.}
$\sum_{t=1}^{n}C_{t}(y_t, l_t) \leq{} \sum_{t=1}^{n}\En_{k\sim{}u_{\gamma, y_{t}}}\brk*{C_{t}(y_t, k)} + S$.
\end{definition}

\subsection{AdaBoost.OLM++}
The high level idea of our algorithm is similar to that of AdaBoost.OL and AdaBoost.OLM:
find a weighted combination of weak learners to minimize some version of the logistic loss in an online manner.
The key difference is that previous works use simple gradient descent to find the weight for each weak learner via proper learning,
while we translate the problem into the framework discussed in \pref{sec:logistic} and deploy the proposed improper learning techniques 
to obtain an improvement on the regret for learning these weights, which then leads to better and in fact optimal sample complexity.

Another difference compared to~\citep{jung2017onlinemulticlass} is that the logistic loss we use
here is more suitable for the multiclass problem than the one they use.\footnote{%
The loss~\citet{jung2017onlinemulticlass} use moves the sum over the incorrect classes outside the log, that is, $\ls(z, y) = \sum_{k\neq y}\log\prn*{1 + e^{z_k-z_y}}$.
}
This simple modification leads to exponential improvement in the number of classes $K$ for the number of weak learners required.

We now describe our algorithm, called AdaBoost.OLM++, in more detail (see \pref{alg:boosting_multiclass} in \pref{app:online_boosting}). 
We denote the $i$-th weak learner as $\wl^{i}$, which is seen as a stateful object and supports two operations:
$\wl^{i}.\predict(x, C)$ predicts a class given an instance and a cost matrix but does not update its internal state;
$\wl^{i}.\update(x, C, y)$ updates the state given an instance, a cost matrix and the true class $y$. 
To keep track of the state we use the notation $\wl^{i}_{t}$ to imply that it has been updated for $t-1$ times.

For each weak learner, the algorithm also maintains an instance of \pref{alg:mixing_multiclass}, denoted by $\logistic^i$,
to improperly learn the aforementioned weight for this weak learner.
Similarly, we use $\logistic^i.\predict(x)$ to denote the prediction step (step 4) in \pref{alg:mixing_multiclass} and $\logistic^i.\update(x, y)$ to denote the update step (i.e. step 5).
The notation $\logistic^i_t$ again implies that the state has been updated for $t-1$ times.

Our algorithm maintains a variable $s_t^i \in \bbR^{K}$ which stands for the weighted accumulated scores of the first $i$ weak learners for instance $x_t$.
When updating $s_t^{i}$ from $s_t^{i-1}$ given the prediction $l_t^i \in [K]$ of weak learner $i$,
our goal is to have the total loss $\sum_{t=1}^n \ls(s_t^{i}, y_t)$ close to $\sum_{t=1}^n \ls(s_t^{i-1}+\alpha e_{l_t^i} , y_t)$
for the best $\alpha$ within some range ($\brk*{-2,2}$ suffices).
Previous works therefore try to learn this weight $\alpha$ via standard online learning approaches.
However, realizing $s_t^{i-1}+\alpha e_{l_t^i}$ can be written as $W\wt{x}_{t}^{i}$
for $W = (\alpha{}I_{K\times{}K}, I_{K\times{}K}) \in \bbR^{K\times{}2K}$
and $\wt{x}_{t}^{i} = (e_{l_{t}^{i}}, s_{t}^{i-1})\in\bbR^{2K}$,
in light of \pref{thm:multiclass_logistic_regret}
we can in fact apply \pref{alg:mixing_multiclass} to learn $s_t^{i}$ if we let the decision set be
$
\cW = \crl*{ (\alpha{}I_{K\times{}K}, I_{K\times{}K}) \in \bbR^{K\times{}2K}\mid{} \alpha\in\brk*{-2,2}}.
$
To make sure that $\wt{x}_{t}^{i}$ has bounded norm, we also set the smoothing parameter $\mu$ to be $1/n$.

With the weighted score $s_t^i$, the prediction coming from the first $i$ weak learner is naturally define as $\yh_t^i = \argmax_{k}s_{t}^{i}(k)$,
the class with the largest score. 
As in AdaBoost.OL and AdaBoost.OLM, 
these predictions $(\yh_t^i)_{i\leq N}$ are treated as $N$ experts and 
the final prediction $y_t$ is determined by the classic Hedge algorithm~\citep{freund1997decision} 
over these experts (Lines~\ref*{line:sample} and~\ref*{line:multiplicative_weights}).

Finally, the cost matrices fed to the weak learners are closely related to the gradient of the loss function.
Formally, define the auxiliary cost matrix $\hC_{t}^{i}$ such that $\hC_{t}^{i}(y, k) = \frac{\partial \ls(z, y)}{\partial z_k} |_{z=s_t^{i-1}}$,
which is simply $\mb{\sigma}(s_{t}^{i-1})_k$ for $k\neq y$ and $\mb{\sigma}(s_{t}^{i-1})_y - 1$ otherwise.
The actual cost matrix is then a translated and scaled version of $\hC_{t}^{i}(y, k)$ so that it belongs to the class $\cC$:
\begin{equation}
\label{eq:boosting_cost_matrix}
C_{t}^{i}(y, k) = \frac{1}{K}\left(\hC_{t}^{i}(y, k) - \hC_{t}^{i}(y, y)\right) \in \cC.
\end{equation}

We now give a mistake bound for AdaBoost.OLM++, which holds even without the weak learning condition and is adaptive to the empirical edge of the weak learners.\footnote{We use notation $\tilde{O}$ and $\tilde{\Omega}$ to hide dependence logarithmic in $n, N, K$ and $1/\delta$.
} All proofs in this section appear in \pref{app:online_boosting}.
\begin{theorem}
\label{thm:multiclass_boosting}
With probability at least $1-\delta$, the predictions $(\yh_t)_{t\leq{}n}$ generated by \pref{alg:boosting_multiclass} satisfy
\begin{equation}
\label{eq:boosting_regret}
\sum_{t=1}^{n}\ind\crl*{\yh_{t}\neq{}y_t} 
= \tilde{O}\prn*{\frac{n}{\sum_{i=1}^{N}\gamma_{i}^{2}}  + \frac{N}{\sum_{i=1}^{N}\gamma_{i}^{2}}},
\end{equation}
where $\gamma_i = \frac{\sum_{t=1}^{n}\hC_{t}^{i}(y_t, l_t^{i})}{\sum_{t=1}^{n}\hC_{t}^{i}(y_t, y_t)} \in [-1, 1]$ is the empirical edge of weak learner $i$.
\end{theorem}

We can now relate the empirical edges to the edge defined in the weak learning condition.
\begin{proposition}
\label{prop:weak_learning_edge}
Suppose all weak learners satisfy the weak learning condition with edge $\gamma$ and sample complexity $S$ (\pref{def:WLC}). 
Then with probability at least $1-\delta$, the predictions $(\yh_t)_{t\leq{}n}$ generated by \pref{alg:boosting_multiclass} satisfy
\begin{equation}
\label{eq:boosting_regret_weak}
\sum_{t=1}^{n}\ind\crl*{\yh_{t}\neq{}y_t} 
= \tilde{O}\prn*{\frac{n}{N\gamma^{2}}  + \frac{1}{\gamma^{2}} + \frac{KS}{\gamma}}.
\end{equation}
Thus, to achieve a target error rate $\veps$, it suffices to take $N=\tilde{\Omega}\prn*{\frac{1}{\veps\gamma^{2}}}$ and 
$n = \tilde{\Omega}(\frac{1}{\veps\gamma^{2}} +  \frac{KS}{\veps\gamma})$.
\end{proposition}

\paragraph{Comparison with prior algorithms}
Compared to~\citep{jung2017onlinemulticlass},
our sample complexity on $n$ improves the dependence on $K$ (for OnlineMBBM) and also $\veps$ and $\gamma$ (for AdaBoost.OLM),
and is in fact optimal according to their lower bound (Theorem~4).
Our bound on the number of weak learners, on the other hand, 
is weaker compared to the non-adaptive algorithm OnlineMBBM (which has a logarithmic dependence on $1/\veps$),
but is still much stronger than that of AdaBoost.OLM since it improves the dependence on $K$ from linear to $\log(K)$.
Although not stated explicitly, our results also apply to the binary setting considered in~\citep{beygelzimer2015optimal}
and improve the sample complexity of their AdaBoost.OL algorithm to the optimal bound $\tilde{\Omega}(\frac{1}{\veps\gamma^{2}} +  \frac{S}{\veps\gamma})$.
Overall, our results significantly reduce the gap between optimal and adaptive online boosting algorithms.

As a final remark, the same technique used here also readily applies to the online boosting setting for the multi-label ranking problem
recently studied by~\citet{jung2017online}. Details are omitted.


\section{High-Probability Online-to-Batch Conversion}
\label{sec:online_to_batch}

Before the present work, the issue of improving on the $O(e^{B})$ fast rate for logistic regression was not addressed even in the batch statistical learning setting. This is perhaps not surprising since the proper lower bound proven by \cite{hazan2014logistic} applies in this setting as well.

Using our improved online algorithm as a starting point, we will show that it is possible to obtain a predictor with excess risk bounded in \emph{high-probability} by $O(d\log(Bn)/n)$ for the batch logistic regression problem. While it is quite straightforward to show that the standard online-to-batch conversion technique applied to \pref{alg:mixing_multiclass} provides a predictor that obtains such an excess risk bound in expectation, obtaining a high-probability bound is far less trivial, as we must ensure that deviations scale at most as $O(\log(B))$. Indeed, a different algorithm is necessary, and our approach is to use a modified version of the ``boosting the confidence'' scheme proposed by \cite{mehta2016fast} for exp-concave losses. Our main result for linear classes is \pref{thm:high_prob_logn} below. For notational convenience will use the shorthand $\En_{(x, y)}[\cdot]$ to denote $\En_{(x, y) \sim \cD}[\cdot]$ where $\cD$ is an unknown distribution over $\cX \times [K]$.

\begin{theorem}[High-probability excess risk bound]
\label{thm:high_prob_logn}
Let $\cD$ be an unknown distribution over $\cX \times [K]$. For any $\delta > 0$ and $n$ samples $\{(x_t, y_t)\}_{t=1}^n$ drawn from $\cD$, we can construct  $g: \cX \rightarrow \bbR^{K}$ such that w.p. at least $1-\delta$, the excess risk $\En_{(x, y)}[\ls(g(x), y)]-\inf_{W\in\cW} \En_{(x, y)}\brk*{\ls(Wx, y)}$ is bounded by
\begin{align*}
O\prn*{\frac{dK\log\prn*{\frac{BRn}{\log(1/\delta)dK} + e}\log\prn*{\frac{1}{\delta}} + \log(Kn)\log\prn*{\frac{\log(n)}{\delta}}}{n}}.
\end{align*}
\end{theorem}
\pref{thm:high_prob_logn} is a consequence of the more general \pref{thm:o2b_high_prob}---stated and proved in \pref{app:o2b_proof}---concerning prediction with the log loss $\ls_\text{log}: \Delta_K \times [K] \rightarrow \R$ defined as $\ls_\text{log}(p, y) = -\log(p_y)$. The theorem asserts that we can convert any online algorithm for multiclass learning with log loss that predicts distributions in $\Delta_K$ for any given input into a predictor for the batch problem with an excess bound essentially equal to the average regret with high probability.


\section{Beyond Linear Classes}
\label{sec:general_class}

We now turn to the question of extending our techniques to general, non-linear predictors. We characterize the minimax regret for learning with the unweighted multiclass logistic loss\footnote{We only consider the unweighted case in this section to avoid excessive notation.} for a general class $\F$ of predictors $f: \cX \rightarrow \R^K$ and abstract instance space $\cX$. This is the same setting as in \pref{sec:prelims}, but with the benchmark class $\crl*{x\mapsto{}Wx\mid{}W\in\cW}$ replaced with an arbitrary class $\cF$, where the loss of a predictor $f\in\cF$ on an example $(x, y) \in \cX \times [K]$ is given by $\ls(f(x), y) = -\log(\mb{\sigma}(f(x))_y)$. The bounds we present in this section---based on sequential covering numbers---substantially increase the scope of results from earlier sections. We note however that they are purely information-theoretic results in the vein of \cite{RakSriTew14jmlr, RakSri14a,RakSri15}, not algorithmic.

Recall that the minimax regret---the best regret bound achievable against the worst-case adaptively chosen sequence of examples---is given by
\begin{equation} \label{eq:minimax_def}
\mathcal{V}_n(\F) = \dtri*{\sup_{x_t\in\cX} \inf_{\zh_t\in\bbR^{K}} \max_{y_t \in [K]}}_{t=1}^n\left[ \sum_{t=1}^n \ell(\zh_t,y_t) - \inf_{f \in \mathcal{F}} \sum_{t=1}^n \ell(f(x_t),y_t)\right],
\end{equation}
where, following \cite{RakSriTew14jmlr}, the $\dtri*{\star}_{t=1}^n$ notation indicates sequential application of the operators contained within $n$ times.

Our bounds on $\cV_n(\cF)$ exploit that the logistic loss can be viewed in two complementary ways: since the loss is $1$-mixable, one can attain a bound of $O(\log |\F|)$ for finite function classes $\F$ using the Aggregating Algorithm, and since the loss is $2$-Lipschitz (in the $\ell_\infty$ norm), for more complex classes one can obtain bounds using sequential complexity measures such as sequential Rademacher complexity \citep{RakSriTew14jmlr}. Our analysis uses both properties simultaneously.

Here is a sketch of the idea for a special case in which we make the simplifying assumption that $\cF$ admits a pointwise cover. 
Recall that a pointwise cover for $\F$ at scale $\gamma$ is a set $V$ of functions $g: \cX \rightarrow \R^K$ such that for any $f \in \F$, there is a $g \in V$ such that for all $x \in \X$, $\|f(x) - g(x)\|_{\infty} \leq \gamma$. Let $N(\gamma)$ be the size of a minimal such cover. For every $g \in V$, let $\F_g = \{f \in \F\mid \sup_{x\in\cX}\nrm*{f(x)-g(x)}_{\infty} \leq \gamma\}$. Now consider the following two-level algorithm. Within each $\F_g$, run the minimax online learning algorithm for this set, then aggregate the predictions for these algorithms over all $g \in V$ using the Aggregating Algorithm to produce the final prediction $\zh_t$.

For each $g \in V$, the regret of the minimax optimal online learning algorithm competing with $\F_g$ can be bounded by the sequential Rademacher complexity of $\F_g$, which can in turn be bounded by the Dudley integral complexity using that the loss is $2$-Lipschitz and that the $L_\infty$ ``radius'' of $\F_g$ is at most $\gamma$ \citep{RakSriTew14jmlr}. The Aggregating Algorithm, via $1$-mixability, ensures a regret bound of $\log N(\gamma)$ against any sub-algorithm. This algorithm has the following regret bound:
\begin{align}
  \label{eq:chaining_simple}
\sum_{t=1}^{n}\ls(\zh_t, y_t) - \inf_{f \in \F}\sum_{t=1}^{n}\ls(f(x_t), y_t) \le \inf_{\gamma >0}  \left\{ \log N(\gamma) + \inf_{\alpha > 0}\left\{8 \alpha n + 24 \sqrt{n} \int_{\alpha}^\gamma \sqrt{\log N(\delta)} d \delta  \right\} \right\}.
\end{align}
This procedure already yields the same bound for the $d$-dimensional linear setting explored earlier: For a class $x\mapsto{}Wx$ with $\nrm*{W}_{2}\leq{}B$ it holds that $N(\gamma) \leq{} \left(\frac{B}{\gamma}\right)^{Kd}$, and we can use this bound in conjunction with \pref{eq:chaining_simple} and the setting $\alpha = \gamma = 1/n$ to get the desired regret bound of $O(dK \log(Bn/dK))$ on the minimax regret.

Unfortunately, this simple approach fails on classes $\F$ for which the pointwise cover is infinite. This can happen for well-behaved function classes that have small \emph{sequential covering number}, even though bounded sequential covering number is sufficient for learnability in the online setting \citep{RakSriTew14jmlr}. We now provide a bound that replaces the pointwise covering number in the argument above with the sequential covering number. The definition of the $L_2$ covering number $\mathcal{N}_2(\alpha,\ell \circ \F)$ that appears in the statement of the theorem below is based on a multiclass generalization of a sequential cover and appears in \pref{app:general_class} due to space limitations.
\begin{theorem}
  \label{thm:logistic_minimax}
Any function class $\cF$ that is uniformly bounded\footnote{Boundedness is required to apply the minimax theorem, but does not explicitly enter our quantitative bounds.} over $\cX$ enjoys the minimax value bound:
\begin{equation}
\label{eq:logistic_minimax}
\cV_{n}(\cF) \leq{} \inf_{\gamma >0}  \left\{ \log \mathcal{N}_2(\gamma, \ell \circ \mathcal{F}) + \inf_{\gamma\geq \alpha > 0}\left\{8 \alpha n + 24 \sqrt{n} \int_{\alpha}^\gamma \sqrt{\log \prn*{\mathcal{N}_2(\delta,\ell \circ \mathcal{F})\cdot{}n}} d \delta  \right\} \right\} + 4.
\end{equation}
\end{theorem}
This rate overcomes several shortcomings faced when trying to apply previously developed minimax bounds for general function classes to the logistic loss. Specifically, \cite{RakSriTew14jmlr} applies to our logistic loss setup but ignores the curvature of the loss and so cannot obtain fast rates, while \cite{RakSri15} obtain fast rates but scale with $e^{B}$, where $B$ is a bound on the magnitude of the predictions, because they use exp-concavity.

Our general function class bound is especially interesting in light of rates obtained in \cite{RakSri14a} for the square loss, which are also based on sequential covering numbers. In the binary case the bound \pref{eq:logistic_minimax} precisely matches the general class bound of \cite[Lemma 5]{RakSri14a} in terms of dependence on the sequential metric entropy. However, \pref{eq:logistic_minimax} does not depend on $B$ explicitly, whereas their Lemma 5 bound for the  square loss explicitly scales with $B^{2}$. In other words, compared to other common curved losses the logistic loss has a desirable property:
\begin{center}\emph{The minimax rate for logistic regression only depends on scale through capacity of the class $\cF$.}\end{center}

Let us examine some rates obtained from this bound for concrete settings. These examples are based on sequential covering bounds that appeared in \cite{RakSri14a,RakSri15}.

\begin{example}[Sparse linear predictors]
Let $\mathcal{G} = \{g_1,\ldots, g_M \}$ be a set of $M$ functions $g_i:\mathcal{X} \mapsto [-B,B]$. Define $\F$ to be the set of all convex combinations of at most $s$ out of these $M$ functions.  The sequential covering number can be easily upper bounded: We can choose $s$ out of $M$ functions in ${M \choose s}$ ways. For each choice, the sequential covering number for the set of all convex combinations of these $s$ bounded functions at scale $\beta$ is bounded as $\frac{B^s}{\beta^s}$. Hence, using that the logistic loss is Lipschitz, we conclude that $\mathcal{N}_2(\F,\beta) = O\prn*{\left(\frac{e M}{s}\right)^s \cdot \beta^{-s } B^s}$. Using this bound with \pref{thm:logistic_minimax} we obtain $\mathcal{V}_n(\F) \le O\left(s \log(BMn/s)\right)$.
\end{example}
The bounds from \cite{RakSriTew14jmlr, RakSri14a,RakSri15} either pay $O(B\sqrt{n})$ or $O(e^{B})$ on this example, whereas the new bound from \pref{eq:logistic_minimax} correctly obtains $O(\log(B))$ scaling.

\begin{example}[Besov classes]
Let $\X$ be a compact subset of $\mathbb{R}^d$. Let $\F$ be the ball of radius $B$ in Besov space $B^s_{p,q}(\X)$. When $s > d/p$ it can be shown that the pointwise log covering number of the space at scale $\beta$ is of order $(B/\beta)^{d/s}$. When $p \ge 2$ one can obtain a sequential covering number bound of order $(B/\beta)^p$ \citep[Section 5.8]{RakSri15a}. These bounds imply:
\begin{enumerate}
 \item If $s \ge d/2$, then $\mathcal{V}_n(\F) \le \tilde{O}\left(B^{\frac{2d}{d + 2s}} n^{\frac{d}{d + 2s}} \right)$.
 \item $s < d/2$, then: if $p > 1 + d/2s$ then $\mathcal{V}_n(\F) \le \tilde{O}\left(B n^{1 - \frac{s}{d}} \right)$; if not, $\mathcal{V}_n(\F)  \le \tilde{O}(B n^{1 - 1/p})$.
\end{enumerate}
\end{example}
\begin{remark}
Using the machinery from the previous section, we can generically lift the general function class bounds given by \pref{thm:logistic_minimax} to high-probability bounds for the i.i.d. batch setting.
\end{remark}


\section{General Function Class Bounds for Log Loss}
\label{sec:log_loss}

In this section we show that our analysis techniques 
can also be used to obtain improved rates for prediction with the \emph{log loss} $\logloss:\Delta_{K}\times{}\brk*{K}\to\bbR$, defined via $\logloss(p, y)=-\log(p_{y})$. Characterizing optimal rates for online prediction with the log loss is a fundamental problem \citep{mf-up-98}, but there have been very few successful attempts to provide rates for general classes of functions. \cite{CesLug99} studied the multiclass case,\footnote{In literature on log loss the class size $K$ we use is typically referred to as the \emph{alphabet size}.} but provide bounds only in terms of pointwise covering numbers; this can lead to vacuous bounds even for well-behaved classes such as Hilbert spaces. More recently, \cite{RakSri15} provided a bound for general classes in terms of sequential covering numbers, but their bound is known to not be tight for certain classes (see the discussion in their Section 6). We improve on their rates uniformly.

Note that the problems of learning with the logistic loss and learning with the log loss can easily be mapped onto each other to provide coarse rates.
One can trivially write $\logloss(p,y)$ as 
$\ell(\mb{\sigma}^{+}(p),y)$ for any distribution $p \in \Delta_K$, and likewise it holds that $\ls(z,y)=\logloss(\mb{\sigma}(z),y)$ for any $z\in\bbR^{K}$. To obtain rates for competing with a class $\cF:\cX\to\Delta_{K}$ under the log loss, we can use this relationship to get a bound by applying \pref{thm:logistic_minimax} with the class $\mb{\sigma}^{+} \circ \F$. This bound improves over \cite{RakSri15} in the low complexity regime, though it is worse for high complexity classes.

By combining the style of proof in \pref{thm:logistic_minimax} with key technical observations from \cite{RakSri15}, we provide a bound on minimax rate for log loss that both uniformly improves on the rate in \cite{RakSri15} for binary outcome case and also extends in general to $K>2$. For brevity we present results only for the binary case. In this case we can restrict to real-valued outputs: We let $\logloss:\brk*{0,1}\times{}\crl*{0,1}\to\bbR$ be defined by $\logloss(p, y) = -y\log(p)-(1-y)\log(1-p)$, and take both $\cF$ and the learner's predictions to be $\brk*{0,1}$-valued. The minimax regret for learning with the log loss is given by

\begin{equation} \label{eq:minimax_logloss}
\mathcal{V}^{\mathrm{log}}_n(\F) = \dtri*{\sup_{x_t\in\cX} \inf_{\lpred_t\in\brk*{0,1}} \max_{y_t \in \crl*{0,1}}}_{t=1}^n\left[ \sum_{t=1}^n \ell(\lpred_t,y_t) - \inf_{f \in \mathcal{F}} \sum_{t=1}^n \ell(f(x_t),y_t)\right].
\end{equation}

The following theorem provides an upper bound on the minimax regret in terms of $L_\infty$ covering numbers $\mathcal{N}_{\infty}(\alpha,\F)$ (definition deferred to \pref{app:logloss}).
\begin{theorem}
\label{thm:logloss_minimax}
For any class $\cF\subseteq{}\brk*{0,1}^{\cX}$ and any $\delta\in(0,1/2]$, $\mathcal{V}^{\mathrm{log}}_n(\F)$ is bounded by
\[
\tilde{O}\prn*{\inf_{\gamma\geq{}\alpha>0}\crl*{ \log\cN_{\infty}(\gamma, \cF)
    +     \frac{\alpha{}n}{\delta}
      + \sqrt{\frac{n}{\delta}}\int_{\alpha}^{\gamma}\sqrt{\log\cN_{\infty}(\rho,\cF)}d\rho
      + \frac{1}{\delta}\int_{\alpha}^{\gamma}\log\cN_{\infty}(\rho,\F)d\rho}
     + \delta{}n
    }.\]
where $\tilde{O}$ supresses $\log(n)$ and $\log(1/\delta)$ factors.
\end{theorem}
Comparing to \cite[Theorem 4]{RakSri15}, the only difference is that their bound has an extra $\frac{1}{\delta}$ factor in the leading $\log\cN_{\infty}(\gamma,\cF)$ term above. \pref{thm:logloss_minimax} is strictly better for low-complexity classes, e.g. when $\log\cN_{\infty}(\gamma,\cF)\asymp\prn*{\frac{C}{\gamma}}^{p}$ for $p\leq{}1$.

 \section{Discussion}
 \label{sec:conclusions}

 We have shown that the simple observation that the logistic loss is $1$-mixable opens the door to significant improvements for various applications via an improper learning algorithm based on Vovk's Aggregating Algorithm, thereby providing solutions to a number of open problems. An important research question left open from this work is that of a truly efficient implemention. While the core algorithm described in this paper can be implemented in polynomial time, it is not a practical algorithm. Obtaining a truly practical algorithm with a modest polynomial dependence on the dimension would be a significant achievement. There is precedent for this kind of algorithm: the Online Newton Step algorithm of \citet{hazan2007logarithmic} was developed as a practically efficient alternative to Cover's Universal Portfolios algorithm, which can also be viewed as an instance of the Aggregating Algorithm.

\subsection*{Acknowledgements}
We thank Sham Kakade for pointing out the connection to Bayesian model averaging. DF thanks Matus Telgarsky for sparking an interest in logistic regression through a series of talks at the Simons Institute. KS acknowledges support from the NSF under grants CDS\&E-MSS 1521544 and NSF CAREER Award 1750575.  MM acknowledges support under NSF grants CCF-1535987 and IIS-1618662. DF is supported in part by the NDSEG PhD fellowship.

{\small
\bibliography{refs}
}

\appendix

\section{Proofs}
\label{app:proofs}

\subsection{Proofs from \pref{sec:logistic}}

\begin{lemma}
\label{lem:multiclass_lipschitz}
The generalized multiclass logisitic loss is $2L$-Lipschitz with respect to $\ls_{\infty}$ norm.
\end{lemma}
\begin{proof}
It is straightforward to verify the identity
\[
\grad{}_z\ls(z,y) = \prn*{\sum_{k}y_k}\mb{\sigma}(z) - y.
\]
It follows that $\nrm*{\grad{}_z\ls(z,y)}_{1}\leq{}\nrm*{y}_{1}\nrm*{\mb{\sigma}(z)}_{1} + \nrm*{y}_1\leq{}2L$. By duality, this implies $2L$-Lipschitzness with respect to $\ls_{\infty}$.
\end{proof}

\begin{lemma}
\label{lem:monomial_concave}
The function $f(x) = \prod_{k\in\brk{d}}x_{k}^{\alpha_{k}}$ is concave over $\bbR_{+}^{d}$ whenever $\alpha_k\geq{}0\;\forall{}k$ and $\sum_{k\in\brk{d}}\alpha_{k}\leq{}1$.
\end{lemma}
\begin{proof}
We will prove that the Hessian of $f$ is negative semidefinite. The Hessian can be written as
\[
\grad^{2}f(x) = f(x)\cdot{}G(x),
\]
where the matrix $G(x)\in\bbR^{d\times{}d}$ is given by $G(x)_{ii} = \alpha_i(\alpha_i-1)x_{i}^{-2}$ and $G(x)_{ij}=\alpha_i\alpha_jx_i^{-1}x_j^{-1}$. Since $f$ is nonnegative, it suffices to show that $G$ is negative semidefinite. Using the reparameterization $y_i=x_i^{-1}$ and the notation $\had$ for the element-wise product, we can write
\[
G(y) = (\alpha\had{}y)^{\tens{}2} - \diag(\alpha\had{}y^{2}).
\]
For any fixed $y\in\bbR^{d}_{+}$ and any $v\in\bbR^{d}$, we have
\begin{align*}
\tri*{v, G(y)v} &= \prn*{\sum_{k=1}^{d}\alpha_ky_kv_k}^{2} - \sum_{k=1}^{d}\alpha_{k}y_k^{2}v_k^{2} \\
&\leq{} \prn*{\sum_{k=1}^{d}\alpha_ky_k^{2}v_k^{2}}\prn*{\sum_{k=1}^{d}\alpha_k} - \sum_{k=1}^{d}\alpha_{k}y_k^{2}v_k^{2}\\
&\leq{} 0.
\end{align*}
The first inequality above uses Cauchy-Schwarz and the second uses that $\sum\alpha_k\leq{}1$.
\end{proof}

\begin{proof}[\pfref{prop:generalized_multiclass_log_mixable}]
We first show that the generalized multiclass log loss $\ls_{\textrm{log}}(p, y)\ldef-\sum_{k\in\brk{K}}y_{k}\log(p_k)$ is $1/L$-mixable over predictions $p \in \Delta_K$ and outcomes $y \in \cY$. Recall that to show $\eta$-mixability it is sufficient to demonstrate that $\ls$ is $\eta$-exp-concave with respect to $p$ (e.g. \citep{PLG}) for any $y \in \cY$.

Observe that we have
\[
e^{-\eta{}\ls(p,y)} = \prod_{k\in\brk{K}}p_{k}^{\eta{}y_k}.
\]

When $\eta\leq{}1/L$, we have $\sum_{k\in\brk{K}} \eta{}y_k\leq{}1$. Since $p \in \Delta_K$ and by the definition of $\cY$, \pref{lem:monomial_concave} implies the function $p\mapsto\prod_{k\in\brk{K}}p_{k}^{\eta{}y_k}$ is concave, which proves the result.

Exp-concavity implies that for any distribution ${\tilde{\pi}}$ over $\Delta_K$, the predicition $p_{\tilde{\pi}}=\En_{p\sim{}{\tilde{\pi}}}\brk*{p}$ certifies the inequality
\[
\En_{p \sim {\tilde{\pi}}}[\exp(-\eta{}\ls_{\textrm{log}}(p, y))] \leq \exp(-\eta\ls_{\textrm{log}}(p_{\tilde{\pi}}, y)) 
\quad{}y\in\cY.
\]
Now, turning to the multiclass logistic loss $\ls: \R^K \times \cY \rightarrow \R$ defined as $\ls(z, y) = -\sum_{k\in\brk{K}}y_{k}\log(\mb{\sigma}(z)_{k})$, let $\pi$ be any distribution on $\R^K$. Let $\tilde{\pi}$ be the induced distribution on $\Delta_K$ via the softmax function, i.e. a sample from $\tilde{\pi}$ is generated by sampling $z \sim \pi$ and computing $p = \mb{\sigma}(z)$. Then define $z_\pi = \mb{\sigma}^{+}\prn*{\En_{z \sim \pi}\brk*{\mb{\sigma}(z)}}$. Since $\mb{\sigma}(z_\pi) = \En_{z \sim \pi}\brk*{\mb{\sigma}(z)} = p_{\tilde{\pi}}$ and $\ls(z, y) = \ls_{\textrm{log}}(\mb{\sigma}(z), y)$, the above inequality implies that
\[
\En_{z \sim \pi}[\exp(-\eta{}\ls(z, y))] \leq \exp(-\eta\ls(z_{\pi}, y)) 
\quad{}y\in\cY.
\]
\end{proof}

\begin{lemma}
  \label{lem:logistic_bounded}
  Suppose a strategy $(\zt_{t}))_{t\leq{}n}$ guarantees a regret inequality
  \[
    \sum_{t=1}^{n}\ls(\zt_t,y_t) - \inf_{f\in\cF}\sum_{t=1}^{n}\ls(f(x_t),y_t) \leq{} \mathbf{R}.
  \]
  Then for $0\leq\mu\leq{}1/2$ the strategy $\hat{z}_{t} \ldef \mb{\sigma}^{+}\prn*{\smooth\prn*{\mb\sigma(\zh_t)}}$ guarantees
  \[
    \sum_{t=1}^{n}\ls(\zt_t,y_t) - \inf_{f\in\cF}\sum_{t=1}^{n}\ls(f(x_t),y_t) \leq{} \mathbf{R} + 2\mu\sum_{t=1}^{n}\nrm*{y_t}_{1}.
  \]
  and satisfies $\nrm*{\hat{z}_{t}}_{\infty}\leq\log(K/\mu)$.

\end{lemma}
\begin{proof}[\pfref{lem:logistic_bounded}]

 We write regret as
\begin{align*}
&\sum_{t=1}^{n}\ls(\zh_t,y_t) - \inf_{f\in\cF}\sum_{t=1}^{n}\ls(f(x_t),y_t) \\ &= 
\sum_{t=1}^{n}\ls(\zt_t,y_t) - \inf_{f\in\cF}\sum_{t=1}^{n}\ls(f(x_t),y_t) + \sum_{t=1}^{n}\ls(\zh_t,y_t) - \sum_{t=1}^{n}\ls(\zt_t,y_t)\\
&\leq{} \mathbf{R} + \sum_{t=1}^{n}\ls(\zh_t,y_t) - \sum_{t=1}^{n}\ls(\zt_t,y_t).
\end{align*}
For the last two terms, fix any round $t$ and define $\tilde{p} = \mb{\sigma}(\zt_t)$. Since $\mb{\sigma}(\zh_t) = (1-\mu) \tilde{p} + \mu \mb{1}/K$, we have
\[
\ls(\zh_t,y_t) - \ls(\zt_t,y_t)
= \sum_{k\in\brk{K}}y_{t,k}\log\prn*{
\frac{\tilde{p}_k}{(1-\mu)\tilde{p}_k + \mu/K}
} \leq{} \log\prn*{
\frac{1}{1-\mu}
}\sum_{k\in\brk{K}}y_{t,k} \leq{} 2\mu\nrm*{y_t}_{1}.
\]
The last inequality uses that $\log(1/(1-x))\leq{}2x$ for $x\leq{}1/2$. Summing up over all rounds $t$ gives us the desired regret bound.

To establish boundedness of the predictions, recall that $\mb{\sigma}_{k}^{+}(p) = \log(p_k)$. Letting $p = (1-\mu)\En_{W\sim{}P_t}\brk*{\mb{\sigma}(Wx_t)} + \mu\ones{}/K$, it clearly holds that $p_{k}\geq{}\mu/K$, and so $|\mb{\sigma}_{k}^{+}(p)|\leq{} \log(K/\mu)$.
\end{proof}

\begin{proof}[\pfref{thm:multiclass_logistic_regret}]
Let $\eta=1/L$. Let $\zt_{t}=\mb{\sigma}^{+}\prn*{\En_{W\sim{}P_t}\brk*{\mb{\sigma}(Wx_t)}}$ --- that is, the prediction for the setting $\mu=0$. We will first establish a regret bound for the case $\mu=0$, then reduce the general case to it by approximation.

First observe that due to mixability for $\eta\leq{}1/L$ (from \pref{prop:generalized_multiclass_log_mixable}), we have
\[
\sum_{t=1}^{n}\ls(\zt_t,y_t) \leq{} -\frac{1}{\eta}\sum_{t=1}^{n}\log\prn*{\int_{\cW}\exp(-\eta\ls(Wx_t, y_t))dP_t(W)}.
\]
Let $Z_{t}=\int_{\cW}\exp\prn{-\eta\sum_{s=1}^{t}\ls(Wx_s, y_s)}dW$ with the convention $Z_0=\int_{\cW}dW$. Using the definition of $P_t$, the right-hand-side in the displayed equation above is then equal to
\[
-\frac{1}{\eta}\sum_{t=1}^{n}\log(Z_{t}/Z_{t-1}) = -\frac{1}{\eta}\log(Z_{n}/Z_{0}) = -\frac{1}{\eta}\log\prn*{\int_{\cW}\exp\prn*{-\eta\sum_{t=1}^{n}\ls(Wx_t, y_t)}dW} + \frac{1}{\eta}\log(\vol(\cW))
\]
We will focus on coming up with an upper bound on the term $-\log\prn*{\int_{\cW}\exp\prn*{-\eta\sum_{t=1}^{n}\ls(Wx_t, y_t)}dW}$. Let $W^{\star}=\argmin_{W\in{}\cW}\sum_{t=1}^{n}\ls(Wx_t,y_t)$. Fix $\theta\in[0,1)$ and let $S=\crl*{\theta{}W^{\star} + (1-\theta)W\mid{}W\in{}\cW}\subseteq{}\cW$. To upper bound the negative-log-integral term, we will lower bound the integral appearing inside.
\begin{align*}
&\int_{\cW}\exp\prn*{-\eta\sum_{t=1}^{n}\ls(Wx_t, y_t)}dW \geq \int_{S}\exp\prn*{-\eta\sum_{t=1}^{n}\ls(Wx_t, y_t)}dW.
\intertext{Using a change of variables and noting that since $W\in\bbR^{K\times{}d}$ the Jacobian of the mapping $W\mapsto (1-\theta)W + \theta{}W^{\star}$ has determinant $(1-\theta)^{\wdim}$, the right-hand-side above equals}
&= (1-\theta)^{\wdim}\int_{\cW}\exp\prn*{-\eta\sum_{t=1}^{n}\ls((\theta{}W^{\star} + (1-\theta)W)x_t, y_t)}dW.
\intertext{Observe that $\nrm*{(\theta{}W^{\star} + (1-\theta)W)x_t - W^{\star}x_t}_{\infty} = (1-\theta)\max_{k\in\brk{K}}\abs*{\tri*{W^{\star}_{k}-W_{k},x_t}}\leq{} 2(1-\theta)B\nrm*{x_t}_{\star}$. 
Using this observation with the $2L$-Lipschitzness of $\ls$ with respect to $\ls_{\infty}$ from \pref{lem:multiclass_lipschitz} implies that the above displayed expression is at most
}
&(1-\theta)^{D_{\cW}}\int_{\cW}\exp\prn*{-\eta\sum_{t=1}^{n}\ls(W^{\star}x_t, y_t) - 4(1-\theta)BL\eta\sum_{t=1}^{n}\nrm*{x_t}_{\star}}dW. \\
&= (1-\theta)^{D_{\cW}}\cdot\mathrm{Vol}(\cW)\cdot\exp\prn*{-\eta\sum_{t=1}^{n}\ls(W^{\star}x_t, y_t)}\cdot\exp\prn*{ - 4(1-\theta)BL\eta\sum_{t=1}^{n}\nrm*{x_t}_{\star}}.
\end{align*}

Combining all of the observations so far, we have proven the following regret bound:
\begin{align*}
&\sum_{t=1}^{n}\ls(\yh_t,y_t) - \sum_{t=1}^{n}\ls(W^{\star}x_t,y_t) \\ &
\begin{aligned}
\leq{}& \frac{1}{\eta}\log(\vol(\mathcal{W})) - \sum_{t=1}^{n}\ls(W^{\star}x_t,y_t) \\ &+ \frac{1}{\eta}\underbrace{\prn*{
D_{\cW}\log\prn*{\frac{1}{1-\theta}} - \log(\vol(\mathcal{W}))
+ \eta\sum_{t=1}^{n}\ls(W^{\star}x_t,y_t)
+ 4(1-\theta)BL\eta\sum_{t=1}^{n}\nrm*{x_t}_{\star}
}}_{\text{Bound on negative log-integral-exp.}}
\end{aligned}
\\
&= \frac{D_{\cW}}{\eta}\log\prn*{\frac{1}{1-\theta}} + 4(1-\theta)BL\sum_{t=1}^{n}\nrm*{x_t}_{\star}.
\end{align*}
To conclude, we choose $\theta$ to satisfy $1-\theta=\min\crl*{D_{\cW}/(B\sum_{t=1}^{n}\nrm*{x_t}_{\star}), 1}$. Note that regardless of which argument obtains the minimum, we have $4(1-\theta)BL\sum_{t=1}^{n}\nrm*{x_t}_{\star} \leq{} 4D_{\cW}L$. The choice of $\theta$ also means that $\log\prn*{\frac{1}{1-\theta}} = \log\prn*{1\vee{}B\sum_{t=1}^{n}\nrm*{x_t}_{\star}/D_{\cW}}$. This leads to a final bound of
\[
D_{\cW}L\cdot{}\log\prn*{1\vee{}\frac{B\sum_{t=1}^{n}\nrm*{x_t}_{\star}}{D_{\cW}}} + 4D_{\cW}L.
\]
To simplify we upper bound this by
\[
5D_{\cW}L\cdot{}\log\prn*{\frac{B\sum_{t=1}^{n}\nrm*{x_t}_{\star}}{D_{\cW}} + e} = 5D_{\cW}L\cdot{}\log\prn*{\frac{BRn}{D_{\cW}} + e}.
\]

To handle the general case where $\mu>0$ we simply appeal to \pref{lem:logistic_bounded} and use that $\mb{\sigma}(\mb{\sigma}^{+}(p)) = p\;\forall{}p\in\Delta_{K}$.

\end{proof}

We now state the proof of \pref{thm:logb_lower_bound}. This proof is a simple corollary of \pref{thm:margin_lb}, a lower bound on mistakes for online binary classification with a margin. \pref{thm:margin_lb} is proven in the remainder of this section of the appendix. To begin, we need the following definition:
\begin{definition}
Let $\cF:\cX\to\brk{-1,1}$ be some function class. A dataset $(x_1,y_1),\ldots,(x_n,y_n)\in \cup_{t=1}^{n}\cX\times{}\pmo$ is shattered with $\gamma$ margin if there exists $f\in\cF$ such that
\[
f(x_t)y_t\geq{}\gamma.
\]
\end{definition}

\begin{proof}[\pfref{thm:logb_lower_bound}]
Let $\zh_t$ for $t \in [n]$ be the sequence of predictions made by the algorithm for a sequence of examples $(x_t, y_t)$, for $t \in [n]$. It is easy to check that
\begin{align*}
	\sum_{t=1}^{n}\ls_\text{bin}(\zh_t, y_t) \geq{} \log(2)\sum_{t=1}^{n}\ind\crl*{\sgn(\zh_t)\neq{}y_t}.
\end{align*}

Let $1/\gamma=B/\log(n)$. From \pref{thm:margin_lb}, it holds that whenever $\gamma\leq{}O(1/\sqrt{d})$, there exists an adversarial sequence $(x_t, y_t)$, for $t \in [n]$, for which 
\[
\sum_{t=1}^{n}\ind\crl*{\sgn(\yh_t)\neq{}y_t}\geq{}\frac{d}{4}\floor*{\log_{2}\prn*{\frac{1}{5\gamma{}d^{1/2}}}},
\]
and for which the dataset is $\gamma$-shattered by some $w\in\bbR^{d}$ with $\nrm*{w}_2\leq{}1$. Since the dataset is $\gamma$-shattered we also have
\[
\inf_{w:\nrm*{w}_2\leq{}B}\sum_{t=1}^{n}\ls_\text{bin}(\tri*{w,x_t}, y_t)\leq{}\sum_{t=1}^{n}\log(1+e^{-\gamma{}B})=\sum_{t=1}^{n}\log\prn*{1+\frac{1}{n}}\leq{}1.
\]
This yields the desired lower bound on the regret.
\end{proof}

\begin{theorem}
\label{thm:margin_lb}
Fix a margin $\gamma\in(0,\frac{1}{4\sqrt{5d}}]$. Then for any randomized strategy $(\yh_t)_{t\leq{}n}$ there exists an adversary $(x_t)_{t\leq{}n}$, $(y_t)_{t\leq{}n}$ with $\nrm*{x_t}_2\leq{}2$ for which
\begin{equation}
\En\brk*{\sum_{t=1}^{n}\ind\crl*{\sgn(\yh_t)\neq{}y_t}} \geq{} \frac{d}{4}\floor*{\log_{2}\prn*{\frac{1}{5\gamma{}d^{1/2}}}},
\end{equation}
and the data sequence is realizable by a unit vector $w\in\bbR^{d+1}$ with margin $\gamma$.
\end{theorem}
\begin{remark}
This lower bound only applies in the regime where $\frac{1}{\gamma^{2}}\geq{}d$, meaning that it does not contradict the dimension-independent Perceptron bound.
\end{remark}
To prove \pref{thm:margin_lb}, we first state a standard lower bound based on Littlestone's dimension.

\begin{definition}
An $\cX$-valued tree is a sequence of mappings $\mb{x}_{t}:\pmo^{t-1}\to\cX$ for $1\leq{}t\leq{}n$.
\end{definition}
We use the abbreviation of $\mb{x}_{t}(\eps) = \mb{x}_{t}(\eps_{1},\ldots,\eps_{t-1})$ for such a tree, where $\eps\in\pmo^{n}$.

\begin{lemma}
\label{lem:fat_shattering}
Let $\cF:\cX\to\brk{-1,1}$ be some function class. Suppose there exists a $\cX$-valued tree $\mb{x}$ of depth $D_{\gamma}$ such that
\begin{equation}
\label{eq:tree_shattered}
\forall{}\eps\in\pmo^{D_{\gamma}}\;\;\exists{}f\in\cF\;\;\;\textrm{s.t.}\;\;\; f(\mb{x}_t(\eps))\eps_{t}\geq{}\gamma.
\end{equation}
Then
\[
    \inf_{q_1,\ldots,q_n}\sup_{\substack{(x_1, y_1),\ldots, (x_n,y_n)\\\textnormal{separable with $\gamma$ margin}}}\En_{\yh_1\sim{}q_t,\ldots,\yh_n\sim{}q_n}\brk*{
    \sum_{t=1}^{n}\ind\crl*{\yh_t\neq{}y_t}
    } \geq{} \frac{1}{2}\min\crl*{D_{\gamma}, n},
\]
where the infimum and supremum above are understood to range over policies.

\end{lemma}
\begin{proof}[\pfref{lem:fat_shattering}]
Suppose that $n\leq{}D_{\gamma}$. We will sample Rademacher random variables $\eps\in\pmo^{n}$ and play $y_t=\eps_t$ and $x_t=\mb{x}_t(\eps_{1:t-1})$. This immediately implies that the expected number of mistakes is equal to $\frac{n}{2}$. Moreover, since $n\leq{}D_{\gamma}$, the assumption in the statement of the lemma implies that  there exists $f\in\cF$ such that $f(\mb{x}_t(\eps))y_t\geq{}\gamma$, so the data is indeed separable with $\gamma$ margin.

If $n>D_{\gamma}$ we can follow the strategy above, then continue to play $(x_{D_{\gamma}}, y_{D_{\gamma}})$ for all $t>D_{\gamma}$.
\end{proof}

\begin{proof}[\pfref{thm:margin_lb}]
By \pref{lem:fat_shattering} it suffices to exhibit a tree $\mb{x}$ for which \pref{eq:tree_shattered} is satisfied with $D_{\gamma}=\Omega(d\log(1/(\sqrt{d}\gamma)))$.

We first restate a well-known tree instance for the one-dimensional case. Consider a class of thresholds $\cF_{\textrm{thresh}}=\crl*{f_{\theta}:\brk*{0,1}\to\pmo}$ defined by $f_{\theta}(z)=1-2\ind\crl*{x<\theta}$. The claim is as follows: For any $\delta\in(0,1]$, there exists a $\brk*{0,1}$-valued tree $\mb{z}$ of depth $D_{\delta}\ldef\floor*{\log_{2}(2/\delta)}$ such that
\begin{enumerate}
\item 
$\forall{}\eps\in\pmo^{D_{\delta}}\;\;\exists{}\theta\;\;\;\textrm{s.t.}\;\;\; f_{\theta}(\mb{z}_t(\eps))\eps_{t}=1$.
\item $\abs*{\mb{z}_t(\eps) - \mb{z}_s(\eps)}\geq{}\delta\;\;\forall{}s\neq{}t$.
\end{enumerate}
The construction is as follows. Let $u_{1}=1$, $l_1=0$. Recursively for $t=1,\ldots,n$:
\begin{itemize}
\item $\mb{z}_{t}(\eps_{1:t-1})=\frac{l_t + u_t}{2}$. 
\item If $\eps_t=-1$ set $l_{t+1}=\mb{z}_t(\eps_{1:t-1})$ and $u_{t+1}=u_t$, else set $u_{t+1}=\mb{z}_t(\eps_{1:t-1})$ and $l_{t+1}=l_t$.
\end{itemize}
Under this construction the sequence $\mb{z}_1(\cdot),\ldots, \mb{z}_{D_{\delta}}(\eps_{1:D_{\delta}-1})$ can always be shattered. Furthermore $\mb{z}^{\star}(\eps)\ldef\mb{z}_{D_{\delta}+1}(\eps_{1:D_{\delta}})$ satisfies the additional property that $\mb{z}_t>\mb{z}^{\star}(\eps)\implies{}\eps_t=1$ and $\mb{z}_t<\mb{z}^{\star}(\eps)\implies\eps_t=-1$. Also, $\abs*{\mb{z}^{\star}-\mb{z}_t}\geq\frac{\delta}{2}\;\forall{}t\leq{}D_{\delta}$.

We now show how to extend this instance to $d+1$ dimensions for any $d\geq{}1$. The approach is to concatenate $d$ instances of the $\mb{z}$ tree constructed above, one for each of the first $d$ coordinates. The final coordinate is left as a constant so that a bias can be implemented.

Let $n=d\cdot{}D_{\delta}$ be the tree depth for our $d+1$-dimensional instance. For any time $t$, let $k\in\brk{d}$ and $\tau\in\brk{D_{\delta}}$ be such that $t=(k-1)D_{\delta} + \tau$. Let any sequence $\eps\in\pmo^{n}$ be partitioned as $(\mb{\eps}^{1},\ldots,\mb{\eps}^{d})$ with each $\mb{\eps}^{k}\in{}\pmo^{D_{\delta}}$. Letting $e_k$ denote the $k$th standard basis vector, we define a shattered tree $\mb{x}$ as follows:
\[
\mb{x}_{t}(\eps_{1:t-1}) = e_{d+1} + e_{k}\mb{z}_{\tau}(\mb{\eps}^{k}_{1:\tau-1}).
\]
We construct a vector $w\in\bbR^{d+1}$ whose sign correctly classifies each $\mb{x}_t$ as follows:
\begin{itemize}
\item $w_{d+1}=-\delta$.
\item $w_{k}=\delta/\mb{z}^{\star}(\mb{\eps}^{k})$.
\end{itemize}
For any $t=(k-1)D_{\delta} + \tau$ this choice gives
\[
\tri*{w, \mb{x}_{t}(\eps)}\eps_t = \delta(\mb{z}_{\tau}(\mb{\eps}^{k}_{1:\tau-1})/\mb{z}^{\star}(\mb{\eps}^{k})-1)\eps_t.
\]
As described above, $\mb{z}_t>\mb{z}^{\star}(\eps)\implies{}\eps_t=1$ and $\mb{z}_t<\mb{z}^{\star}(\eps)\implies\eps_t=-1$, which immediately implies that the inner product is always non-negative, and so the dataset is shattered. Using that $\abs*{\mb{z}^{\star}(\eps)-\mb{z}_t(\eps)}\geq{}\frac{\delta}{2}$ and that both numbers lie in $\brk*{0,1}$, we can lower bound the magnitude with which the shattering takes place:
\[
\abs*{\mb{z}_{\tau}(\eps^{k}_{1:\tau-1})/\mb{z}^{\star}(\mb{\eps}^{k})-1} = \frac{1}{\mb{z}^{\star}(\mb{\eps}^{k})}\abs*{\mb{z}_{\tau}(\eps^{k}_{1:\tau-1})-\mb{z}^{\star}(\mb{\eps}^{k})}\geq{}\frac{1}{\mb{z}^{\star}(\mb{\eps}^{k})}\frac{\delta}{2}\geq{}\frac{\delta}{4},
\]
and so the shattering takes place with margin at least $\delta^{2}/4$.

Lastly, the norm of $w$ is given by
\[
\nrm*{w}_{2}=\sqrt{\delta^{2} + \sum_{k=1}^{d}\prn*{\frac{\delta}{\mb{z}^{\star}(\mb{\eps}^{k})}}^{2}}
\leq{} \sqrt{\delta^{2} + 4d} \leq{} \sqrt{5d},
\]
where the first inequality uses that $\mb{z}^{\star}(\eps)\geq{}\delta/2$ and the second uses that $d\geq{}1$

Rescaling, we have that the vector $w/\nrm*{w}_{2}$ shatters the tree with margin at least $\frac{\delta^{2}}{4\sqrt{5d}}$. To rephrase the result as a function of a desired margin: For any margin $\gamma\in(0,\frac{1}{4\sqrt{5d}}]$, setting $\delta=\sqrt{\gamma4\sqrt{5d}}\leq{}1$, we have constructed a tree of depth $\floor*{\log_{2}(2/\sqrt{\gamma4\sqrt{5d}})}$ that can be shattered with margin $\gamma$.

\end{proof}


\subsection{OBAMA Algorithm and Proof of \pref{thm:bandit_multiclass}}

\begin{algorithm}[h]
\caption{}
\label{alg:bandit_multiclass}
\begin{algorithmic}[1]
\Procedure{OBAMA}{decision set $\cW$, smoothing parameter $\mu$.}
\State Let $\cA$ be \pref{alg:mixing_multiclass} initialized with $\cW$ and $\mu$.
\For{$t=1,\ldots,n$}
\State Obtain $x_t$, pass it to $\cA$ and let $\zh_t \in \R_K$ be the output of $\cA$.
\State Play $\yh_t \sim p_t := \mb{\sigma}(\zh_t)$ and obtain $\ind[\yh_t \neq y_t]$.
\State Define $\tilde{y}_t \in \R^K$ as $\tilde{y}_t(k) := \frac{\ind[k = \yh_t]\ind[\hy_t = y_t]}{p_t(\yh_t)}$ for $k \in [K]$ and pass it as feedback to $\cA$.
\EndFor
\EndProcedure
\end{algorithmic}
\end{algorithm}

\label{app:bandit_multiclass_proofs}
\begin{proof}[Proof of \pref{thm:bandit_multiclass}]
First, note that an easy calculation on the softmax function $\mb{\sigma}$ implies that for all $k \in [K]$, $p_t(k) \geq \frac{(1-\mu) \exp(-2BR) + \mu}{K}$. So, defining $L = \frac{K}{(1-\mu) \exp(-2BR) + \mu}$, we have $\|\tilde{y}_t\|_1 \leq L$. Thus, \pref{thm:multiclass_logistic_regret} applied to $\cA$ guarantees that for any $W \in \cW$, 
\[\sum_{t=1}^{n}\ls(\zh_t,\tilde{y}_t) - \sum_{t=1}^{n}\ls(Wx_t,\tilde{y}_t) \leq{} 5LdK\cdot{}\log\prn*{\tfrac{BRn}{dK} + e} + 2\mu{}\sum_{t=1}^n \|\tilde{y}_t\|_1.\]
Fix a round $t$ and let $\En_t[\cdot]$ denote expecation conditioned on $\yh_1, \yh_2, \ldots, \yh_{t-1}$. The construction of the feedback vectors $\tilde{y}_t$ via importance weighting guarantees $\En_t[\tilde{y}_t] = \mb{1}_{y_t}$, where $\mb{1}_k$ denotes the indicator vector supported on coordinate $k$. Hence, $\En_t[\ell(\zh_t, \tilde{y}_t)] = \ell(\zh_t, y_t) = -\log(p_t(y_t))$ and $\En_t[\ell(Wx_t, \tilde{y}_t)] = \ell(Wx_t, y_t)$. Furthermore, it is easy to check that $\En_t[\|\tilde{y}_t\|_1] = 1$. Thus, we conclude that
\[\sum_{t=1}^{n}\En[-\log(p_t(y_t))] - \sum_{t=1}^{n}\ls(Wx_t,y_t) \leq{} 5LdK\cdot{}\log\prn*{\tfrac{BRn}{dK} + e} + 2\mu{}n.\]

Now if we set $\mu = 0$, then the right-hand side is bounded by $O(dK^2\exp(2BR)\log\prn*{\tfrac{BRn}{dK} + e})$. If we set $\mu = \sqrt{\frac{dK^2\log\prn*{\tfrac{BRn}{dK} + e}}{n}}$, the right-hand side is bounded by $O\left(\sqrt{dK^2\log(\tfrac{BRn}{dK} + e)n}\right)$. Choosing the setting of $\mu$ that gives the smaller upper bound, and the fact that the log loss upper bounds the probability of making a mistake (because $-\log(p_t(y_t)) \geq 1 - p_t(y_t)$), we get the stated bound on the expected number of mistakes.
\end{proof}



\subsection{Pseudocode and Proofs from \pref{sec:online_boosting}}
\label{app:online_boosting}

\begin{algorithm}[h!]
\caption{AdaBoost.OLM++}
\label{alg:boosting_multiclass}
\begin{algorithmic}[1]
\Procedure{AdaBoost.OLM++}{weak learners $\wl^{1}, \ldots, \wl^{N}$}
\State{For all $i\in[N]$, set $v_{1}^{i}\gets{} 1$, initialize weak learner $\wl_{1}^{i}$, and initialize logistic learner $\logistic_{1}^{i}$
with $\cW = \crl*{ (\alpha{}I_{K\times{}K}, I_{K\times{}K}) \in \bbR^{K\times{}2K}\mid{} \alpha\in\brk*{-2,2}}$ and $\mu = 1/n$.}
\For{$t=1,\ldots,n$}
\State{Receive instance $x_{t}$.}
\State{$s_{t}^{0} \gets{} 0\in\bbR^{K}$.}
\For{$i=1,\ldots,N$}
\State{Compute cost matrix $C_{t}^{i}$ from $s_{t}^{i-1}$ using \pref{eq:boosting_cost_matrix}.}
\State{$l_{t}^{i}\gets{}\wl_{t}^{i}.\predict(x_{t}, C_{t}^{i})$.}
\State{$\wt{x}_{t}^{i}\gets{}(e_{l_{t}^{i}}, s_{t}^{i-1})\in\bbR^{2K}$.}
\State{$s_{t}^{i}\gets{}\logistic_{t}^{i}.\predict(\wt{x}_t^{i})$.}
\State{$\yh_{t}^{i}\gets{}\argmax_{k}s_{t}^{i}(k)$.}
\EndFor
\State{Sample $i_{t}$ with $\Pr(i_{t}=i)\propto{}v_{t}^{i}$. \label{line:sample}}
\State{Predict $\yh_{t}=\yh_{t}^{i_t}$ and receive true class $y_{t}\in\brk*{K}$.}
\For{$i=1,\ldots,N$}
\State{$\wl_{t+1}^{i}\gets{}\wl_{t}^{i}.\update(x_t, C_{t}^{i}, y_t)$.}
\State{$\logistic_{t+1}^{i}\gets{}\logistic_{t}^{i}.\update(\wt{x}_{t}^{i}, \mb{1}_{y_t})$.}
\State{$v_{t+1}^{i}\gets{}v_{t}^{i}\cdot\exp\prn*{-\ind\crl*{\yh_{t}^{i}\neq{}y_t}}$.\label{line:multiplicative_weights}}
\EndFor
\EndFor
\EndProcedure
\end{algorithmic}
\end{algorithm}

\begin{proof}[\pfref{thm:multiclass_boosting}] 
Denote the number of mistakes of the $i$-th expert (which is the combination of the first $i$ weak learners) by
\[
M_{i} = \sum_{t=1}^{n}\ind\crl*{\yh_{t}^{i}\neq{}y_t}=\sum_{t=1}^{n}\ind\crl*{\argmax_{k}s_{t}^{i}(k)\neq{}y_t},
\]
with the convention that $M_{0}=n$. The weights $v_{t}^{i}$ simply implement the multiplicative weights strategy, and so \pref{lem:multiplicative_weights_conc}, which gives a concentration bound based on Freedman's inequality implies that with probability at least $1-\delta$,\footnote{%
Note that previous online boosting works \citep{beygelzimer2015optimal,jung2017onlinemulticlass} use a simpler Hoeffding bound at this stage, which picks up an extra $\sqrt{n}$ term. For their results this is not a dominant term, but in our case it can spoil the improvement given by improper logistic regression, and so we use Freedman's inequality to remove it.
}
\begin{equation}
\label{eq:multiplicative_weights_bound}
\sum_{t=1}^{n}\ind\crl*{\yh_{t}\neq{}y_t} \leq{} 4\min_{i}M_{i} + 2\log(N/\delta).
\end{equation}

Note that if $k^{\star}\ldef\argmax_{k}s_{t}^{i-1}(k)\neq{}y_t$, then $\mb{\sigma}(s_{t}^{i-1})_{k^{\star}}\geq{}\mb{\sigma}(s_{t}^{i-1})_{y_t}$ and $\mb{\sigma}(s_{t}^{i-1})\in\Delta_{K}$ imply $\mb{\sigma}(s_{t}^{i-1})_{y_t}\leq{}1/2$, which then implies $\sum_{k\neq{}y_t}\mb{\sigma}(s_{t}^{i-1})_{k}\geq{}1/2$ and finally
\begin{equation}
\label{eq:cost_matrix_mistakes}
-\sum_{t=1}^{n}\hC_{t}^{i}(y_t, y_t) = \sum_{t=1}^n\sum_{k\neq{}y_t}\mb{\sigma}(s_{t}^{i-1})_{k} \geq{} \frac{M_{i-1}}{2}.
\end{equation}
This also holds for $i=1$ because $s_{t}^{0}=0$ and $-C_{t}^{1}(y_t, y_t)=(K-1)/K\geq{}1/2$.

We now examine the regret guarantee provided by each logistic regression instance. For each $i\in\brk{N}$ we have
\begin{equation*}
\sum_{t=1}^{n}\ls(s_{t}^{i},y_t) - \inf_{W\in\cW}\sum_{t=1}^{n}\ls(W\wt{x}^{i}_t,y_t) \leq{} O\prn*{\log\prn*{n\log(nK)}}
\end{equation*}
This follows from \pref{thm:multiclass_logistic_regret} using $L = 1$, $D_\cW = 1$, $B=3$ for $\ls_{1}$ norm, $\nrm*{y_t}_{1}=1$, $\mu=1/n$, and $\nrm*{\wt{x}_{t}^{i}}_{\infty} \leq \log(nK)$,
where the last fact is implied by the second statement of \pref{thm:multiclass_logistic_regret}: $\nrm*{s_{t}^{i}}_{\infty}\leq{}\log(K/\mu) = \log(nK)$ and thus $\nrm*{\wt{x}_{t}^{i}}_{\infty} = \nrm*{(e_{l_{t}^{i}}, s_{t}^{i-1})}_{\infty} \leq{} \log(nK)$. 
Now define the difference between the total loss of the $i$-th and $(i-1)$-th expert to be
\[
\Delta_{i} = \sum_{t=1}^{n}\ls(s_{t}^{i}, y_t) - \ls(s_{t}^{i-1}, y_t).
\]
Since $\inf_{W\in\cW}\sum_{t=1}^{n}\ls(W\wt{x}^{i}_t,y_t) = \inf_{\alpha\in\brk*{-2,2}}\sum_{t=1}^{n}\ls(\alpha{}e_{l_{t}^{i}} + s_{t}^{i-1},y_t)$, the regret bound above implies
\[
\Delta_{i} \leq{} \inf_{\alpha\in\brk*{-2,2}}\brk*{\sum_{t=1}^{n}\ls(\alpha{}e_{l_{t}^{i}} + s_{t}^{i-1},y_t) - \ls(s_{t}^{i-1}, y_t)} + O\prn*{\log\prn*{n\log(nK)}}.
\]
By \pref{lem:logistic_ub} each term in the sum above satisfies
\[
\ls(\alpha{}e_{l_{t}^{i}} + s_{t}^{i-1},y_t) - \ls(s_{t}^{i-1}, y_t)
\leq{} \left\{
\begin{array}{ll}
(e^{\alpha}-1)\mb{\sigma}(s_{t}^{i-1})_{l_{t}^{i}} = (e^{\alpha}-1)\hC_{t}^{i}(y_t, l_{t}^{i}),\quad &l_{t}^{i} \neq{}y_t,\\
(e^{-\alpha}-1)(1-\mb{\sigma}(s_{t}^{i-1})_{y_t}) = -(e^{-\alpha}-1)\hC_{t}^{i}(y_t, y_t),\quad{} &l_{t}^{i} =y_t.
\end{array}
\right.
\]
With notation $w^{i}=-\sum_{t=1}^{n}\hC_{t}^{i}(y_t, y_t)$, $c_{+}^{i}=-\frac{1}{w^i}\sum_{t:l_{t}^{i}=y_t}\hC_{t}^{i}(y_t, y_t)$, and $c_{-}^{i}=\frac{1}{w^i}\sum_{t:l_{t}^{i}\neq{}y_t}\hC_{t}^{i}(y_t, l_t^{i})$, we rewrite 
\[
\inf_{\alpha\in\brk*{-2,2}}\brk*{\sum_{t=1}^{n}\ls(\alpha{}e_{l_{t}^{i}} + s_{t}^{i-1},y_t) - \ls(s_{t}^{i-1}, y_t)} 
= w^{i}\cdot\inf_{\alpha\in\brk*{-2,2}}\brk*{(e^{\alpha}-1)c_{-}^{i} + (e^{-\alpha}-1)c_{+}^{i}}.
\]
One can verify that $w^{i}>0$,  $c_{-}^{i}, c_{+}^{i}\geq{}0$, $c_{+}^{i}-c_{-}^{i}=\gamma_{i}\in\brk*{-1,1}$ and $c_{+}^{i} + c_{-}^{i}\leq{}1$. 
By \pref{lem:logistic_inf}, it follows that
\[
w^{i}\cdot\inf_{\alpha\in\brk*{-2,2}}\brk*{(e^{-\alpha}-1)c_{-}^{i} + (e^{\alpha}-1)c_{+}^{i}}
\leq{} -\frac{w^{i}\gamma_{i}^{2}}{2}.
\]

Summing $\Delta_{i}$ over $i\in [N]$, we have
\begin{equation}
\label{eq:delta_sum}
\sum_{t=1}^{n}\ls(s_{t}^{N}, y_t) - \sum_{t=1}^{n}\ls(s_{t}^{0}, y_t)  = \sum_{i=1}^{N}\Delta_i
\leq{} -\frac{1}{2}\sum_{i=1}^{N}w^{i}\gamma_{i}^{2} + O(N\log(n\log(nK))).
\end{equation}
We lower bound the left hand side as
\[
\sum_{t=1}^{n}\ls(s_{t}^{N}, y_t) - \sum_{t=1}^{n}\ls(s_{t}^{0}, y_t)
\geq{} - \sum_{t=1}^{n}\ls(s_{t}^{0}, y_t) = -n\log(K),
\]
where the inequality uses non-negativity of the logistic loss and the equality is a direct calculation from $s_{t}^{0}=0$.
Next we upper bound the right-hand side of \pref{eq:delta_sum}. Since $w^{i}=-\sum_{t=1}^{n}\hC_{t}^{i}(y_t, y_t)$, Eq.~\pref{eq:cost_matrix_mistakes} implies
\begin{align*}
-\frac{1}{2}\sum_{i=1}^{N}w^{i}\gamma_{i}^{2} 
\leq  -\frac{1}{4}\sum_{i=1}^{N}M_{i-1}\gamma_{i}^{2} 
\leq{} -\min_{i\in\brk*{N}}M_{i-1}\cdot\frac{1}{4}\sum_{i=1}^{N}\gamma_{i}^{2} 
\leq{} -\min_{i\in\brk*{N}}M_{i}\cdot\frac{1}{4}\sum_{i=1}^{N}\gamma_{i}^{2}.
\end{align*}

Combining our upper and lower bounds on $\sum_{i=1}^{N}\Delta_{i}$ now gives
\begin{equation}
\label{eq:multiclass_boosting_combined}
-n\log(K) \leq{} -\frac{1}{2}\sum_{i=1}^{N}w^{i}\gamma_{i}^{2} + O(N\log(n\log(K)))\leq{} -\min_{i\in\brk*{N}}M_{i}\cdot\frac{1}{4}\sum_{i=1}^{N}\gamma_{i}^{2} + O(N\log(n\log(nK))).
\end{equation}
Rearranging, we have
\[
\min_{i\in\brk*{N}}M_{i} \leq{} O\prn*{\frac{n\log(K)}{\sum_{i=1}^{N}\gamma_{i}^{2}}} + O\prn*{\frac{N\log(n\log(nK))}{\sum_{i=1}^{N}\gamma_{i}^{2}}}.
\]
Returning to \pref{eq:multiplicative_weights_bound}, this implies that with probability at least $1-\delta$,
\[
\sum_{t=1}^{n}\ind\crl*{\yh_{t}\neq{}y_t} \leq{} O\prn*{\frac{n\log(K)}{\sum_{i=1}^{N}\gamma_{i}^{2}}} + O\prn*{\frac{N\log(n\log(nK))}{\sum_{i=1}^{N}\gamma_{i}^{2}}} + 2\log(N/\delta),
\]
which finishes the proof.
\end{proof}

\begin{proof}[\pfref{prop:weak_learning_edge}]
By the definition of the cost matrices, the weak learning condition
\[
\sum_{t=1}^{n}C_{t}^{i}(y_t, l_{t}^{i}) \leq{} \sum_{t=1}^{n}\En_{k\sim{}u_{\gamma, y_t}}\brk*{C_{t}^{i}(y_t, k)} + S
\]
implies
\[
\sum_{t=1}^{n}\hC_{t}^{i}(y_t, l_{t}^{i}) \leq{} \sum_{t=1}^{n}\En_{k\sim{}u_{\gamma, y_t}}\brk*{\hC_{t}^{i}(y_t, k)} + KS
\]
Expanding the definitions of $u_{\gamma, y_t}$ and $\hC_{t}^{i}$, we have
\[
\En_{k\sim{}u_{\gamma, y_t}}\brk*{\hC_{t}^{i}(y_t, k)} = \prn*{\frac{1-\gamma}{K}}\prn*{(\mb{\sigma}(s_{t}^{i-1})_{y_t}-1) + \sum_{k\neq{}y_t}\mb{\sigma}(s_{t}^{i-1})_{k}} + \gamma{}(\mb{\sigma}(s_{t}^{i-1})_{y_t}-1) = \gamma{}\hC_{t}^{i}(y_t, y_t).
\]
So we have
\[
\sum_{t=1}^{n}\hC_{t}^{i}(y_t, l_{t}^{i}) \leq{} \gamma\sum_{t=1}^{n}\hC_{t}^{i}(y_t, y_t) + KS, 
\]
or, since $\hC_{t}^{i}(y_t, y_t)<0$,
\[
\gamma_{i} \geq{} \gamma - \frac{KS}{w^{i}},
\]
where $w^{i}=-\sum_{t=1}^{n}C_{t}^{i}(y_t, y_t)$ as in the proof of \pref{thm:multiclass_boosting}. 
Since $a \geq b - c$ implies $a^2 \geq b^2 - 2bc$ for non-negative $a, b$ and $c$, 
we further have $\gamma_{i}^{2}\geq{}\gamma^{2}-2\frac{\gamma{}KS}{w^{i}}$.

Returning to the inequality~\pref{eq:multiclass_boosting_combined}, the bound we just proved implies
\begin{align*}
-n\log(K) &\leq{} -\frac{1}{2}\sum_{i=1}^{N}w^{i}\gamma^{2} + \gamma{}KSN + O(N\log(n\log(nK))) \\
&\leq{} -\frac{\gamma^2}{4}\sum_{i=1}^{N}M_{i-1}+ \gamma{}KSN + O(N\log(n\log(nK)))  \tag{by \pref{eq:cost_matrix_mistakes}}\\
&\leq{} -\min_{i\in\brk*{N}}M_{i} \cdot \frac{\gamma^{2}N}{4} + \gamma{}KSN + O(N\log(n\log(nK))).
\end{align*}
From here we proceed as in the proof of \pref{thm:multiclass_boosting} to get the result.
\end{proof}

\begin{lemma}[Freedman's Inequality \citep{beygelzimer2011contextual}]
\label{lem:freedman}
Let $(Z_t)_{t\leq{}n}$ be a real-valued martingale difference sequence adapted to a filtration $(\cJ_t)_{t\leq{}n}$ with $\abs{Z_t}\leq{}R$ almost surely. For any $\eta\in[0, 1/R]$, with probability at least $1-\delta$,
\begin{equation}
\label{eq:freedman}
\sum_{t=1}^{n}Z_t \leq{} \eta(e-2)\sum_{t=1}^{n}\En\brk*{Z_t^{2} \mid{} \cJ_t} + \frac{\log(1/\delta)}{\eta}
\end{equation}
for all $\eta\in\brk*{0, 1/R}$.
\end{lemma}
\begin{lemma}
\label{lem:multiplicative_weights_conc}
With probability at least $1-\delta$, the predictions $(\yh_t)_{t\leq{}n}$ generated by \pref{alg:boosting_multiclass} satisfy
\[
\sum_{t=1}^{n}\ind\crl*{\yh_{t}\neq{}y_t} \leq{} 4\min_{i}\sum_{t=1}^{n}\ind\crl*{\yh_{t}^{i}\neq{}y_t} + 2\log(N/\delta).
\]
\end{lemma}
\begin{proof}
Define a filtration $(\cJ_{t})_{t\leq{}n}$ via 
\[
\cJ_{t}=\sigma((x_1, (l_{1}^i)_{i\leq{}N}, y_1, i_1),\ldots,(x_{t-1}, (l_{t-1}^i)_{i\leq{}N}, y_{t-1}, i_{t-1}), x_t, (l_{t}^i)_{i\leq{}N}).
\]

Since Line~\ref*{line:multiplicative_weights} of \pref{alg:boosting_multiclass} implements the multiplicative weights strategy with learning rate $1$, the standard analysis (e.g. \cite{PLG}) implies that the conditional expectations under this strategy enjoy a regret bound of
\[
\sum_{t=1}^{n}\En\brk*{\ind\crl*{\yh_{t}\neq{}y_t}\mid{}\cJ_{t}} \leq{} 2\min_{i}\sum_{t=1}^{n}\ind\crl*{\yh_{t}^{i}\neq{}y_t} + \log(N).
\]
Let $Z_{t}=\ind\crl*{\yh_{t}\neq{}y_t}-\En\brk*{\ind\crl*{\yh_{t}\neq{}y_t}\mid{}\cJ_{t}}$. \pref{lem:freedman} applied with $\eta=1$ shows that with probability at least $1-\delta$, 
\[
\sum_{t=1}^{n}Z_t \leq{} \sum_{t=1}^{n}\En\brk*{Z_t^{2} \mid{} \cJ_t} + \log(1/\delta).
\]
Since variance is bounded by second moment, we have
\[
\sum_{t=1}^{n}\En\brk*{Z_t^{2} \mid{} \cJ_t} \leq{} \sum_{t=1}^{n}\En\brk*{(\ind\crl*{\yh_{t}\neq{}y_t})^{2} \mid{} \cJ_t}
= \sum_{t=1}^{n}\En\brk*{\ind\crl*{\yh_{t}\neq{}y_t} \mid{} \cJ_t}.
\]
Rearranging, we have proved that with probability $1-\delta$,
\[
\sum_{t=1}^{n}\ind\crl*{\yh_{t}\neq{}y_t} \leq{}  2\sum_{t=1}^{n}\En\brk*{\ind\crl*{\yh_{t}\neq{}y_t} \mid{} \cJ_t} + \log(1/\delta) \leq{} 
4\min_{i}\sum_{t=1}^{n}\ind\crl*{\yh_{t}^{i}\neq{}y_t} + 2\log(N/\delta).
\]
\end{proof}

\begin{lemma}\label{lem:logistic_ub}
The multiclass logistic loss satisfies for any $z \in \bbR^K$ and $y\in [K]$,
\[
\ls(z + \alpha{}e_{l}, y) - \ls(z,y)
\leq{} \left\{
\begin{array}{ll}
(e^{\alpha}-1)\mb{\sigma}(z)_{l},\quad &l \neq{}y,\\
(e^{-\alpha}-1)(1-\mb{\sigma}(z)_{y}),\quad{} &l=y.
\end{array}
\right.
\]
\end{lemma}
\begin{proof}
When $l\neq{}y$ we have
\begin{align*}
\ls(z + \alpha{}e_{l}, y) - \ls(z,y) &= \log\prn*{
\frac{1+\sum_{k\neq{}y, l}e^{z_k - z_y} + e^{z_l + \alpha - z_y} }
{1+\sum_{k\neq{}y}e^{z_k - z_y}}
} \\
&= \log\prn*{
1 + (e^{\alpha}-1)\frac{e^{z_l-z_y}}{1+\sum_{k\neq{}y}e^{z_k-z_y}}
} \\
&= \log\prn*{
1 + (e^{\alpha}-1)\mb{\sigma}(z)_{l}
} \\
&\leq{} (e^{\alpha}-1)\mb{\sigma}(z)_{l} \tag{$\log(1+x)\leq{}x$}.
\end{align*}

When $l=y$ we have
\begin{align*}
\ls(z + \alpha{}e_{l}, y) - \ls(z,y) &= \log\prn*{
\frac{1+e^{-\alpha}\sum_{k\neq{}y}e^{z_k - z_y}}
{1+\sum_{k\neq{}y}e^{z_k - z_y}}
} \\
&= \log\prn*{
1+(e^{-\alpha}-1)\frac{\sum_{k\neq{}y}e^{z_k - z_y}}{1 + \sum_{k\neq{}y}e^{z_k - z_y}}
} \\
&= \log\prn*{
1 + (e^{-\alpha}-1)\sum_{k\neq{}y}\mb{\sigma}(z)_{k}
} \\
&= \log\prn*{
1 + (e^{-\alpha}-1)(1-\mb{\sigma}(z)_{y})
} \\
&\leq{} (e^{-\alpha}-1)(1-\mb{\sigma}(z)_{y}). \tag{$\log(1+x)\leq{}x$}
\end{align*}
\end{proof}

\begin{lemma}[\cite{jung2017onlinemulticlass}]
\label{lem:logistic_inf}
For any $A,B\geq{}0$ with $A-B\in\brk*{-1, +1}$ and $A+B\leq{}1$,
\[
\inf_{\alpha\in\brk*{-2,2}}\brk*{A(e^{\alpha}-1) + B(e^{-\alpha}-1)} \leq{} -\frac{(A-B)^{2}}{2}.
\]
\end{lemma}


\subsection{Proof from \pref{sec:online_to_batch}}
\label{app:o2b_proof}

\begin{theorem}
\label{thm:o2b_high_prob}
Let $\cF$ be a class of functions $f: \cX \rightarrow \Delta_K$. Suppose there is an online multiclass learning algorithm over $\cF$ using the log loss that for any data sequence $(x_t, y_t) \in \cX \times [K]$ for $t = 1, 2, \ldots, n$ produces distributions $p_t \in \Delta_K$ such that the following regret bound holds:
\[ \sum_{t=1}^n \logloss(p_t, y_t) - \inf_{f \in \cF}\sum_{t=1}^n \logloss(f(x_t), y_t) \leq R(n).\]
Here $R(n)$ is some function of $n$ and other relevant problem dependent parameters. Then for any given $\delta > 0$ and any (unknown) distribution $\cD$ over $\cX \times [K]$, it is possible to construct a predictor $g: \cX \rightarrow \Delta_K$ using $n$ samples $\{(x_t, y_t)\}_{t=1}^n$ drawn from $\cD$ such that with probability at least $1-\delta$, the excess risk of $g$ is bounded as
\[ \En_{(x, y)}[\logloss(g(x), y)] \leq \inf_{f \in \cF} \En_{(x, y)}[\logloss(f(x), y)] + O\prn*{\frac{\log\prn*{\frac{1}{\delta}}R\prn*{\frac{n}{\log(1/\delta)}} + \log(Kn)\log\prn*{\frac{\log(n)}{\delta}}}{n}}.
\]
\end{theorem}

\begin{proof}[Proof of \pref{thm:o2b_high_prob}]
Recall that the standard online-to-batch conversion \citep{helmboldwarmuth} produces an (improper) predictor using $n$ data samples by running the online algorithm on those samples and stopping at a random time. Then predictor is online algorithm with its the internal state frozen. This predictor has excess risk bounded by the average regret over $n$ rounds, in expectation over the $n$ data samples.

The algorithm to generate the predictor $g$ with the specified excess risk bound in the theorem statement is given below:
\begin{enumerate}
\item Let $M=\lceil\log(2/\delta)\rceil$. Produce $M$ predictors $h_{1},\ldots,h_{M}: \cX \rightarrow \Delta_K$ by using the online-to-batch conversion on the online multiclass learning algorithm run using $M$ disjoint sets of $n/2M$ samples each. Call the $i$th such set of samples $S_{i}$
\item For $i \in [M]$, define $\tilde{h}_{i}: \cX \rightarrow \Delta_K$ as $\tilde{h}_i(x) = \smooth\prn*{h_i(x)}$ for $\mu = \frac{R(n/M)}{2n/M}$.
\item Construct an online convex optimization instance as follows. The learner's decision set is $\Delta_M$, the set of all distributions on $[M]$. For every data point $(x, y) \in \cX \times [K]$, associate the loss function $\ls_{(x, y)}: \Delta_M \rightarrow \R$ defined as $\ls_{(x, y)}(q) = -\log(\En_{i \sim q}[(\tilde{h}_{i}(x))_y])$. These loss functions are $1$-exp-concave, so run the EWOO algorithm \citep{hazan2007logarithmic} using the remaining $n/2$ examples sequentially to generate loss functions. Let $\bar{q}$ be the average of all the distributions in $\Delta_M$ generated by EWOO. Define $g := \En_{i \sim \bar{q}}[\tilde{h}_i]$.
\end{enumerate}

We now proceed to analyse the excess risk of $g$. First, using the regret bound for the online multiclass learning algorithm, and in-expecation bound on the excess risk for online-to-batch conversion, for every $i \in [M]$, we have 
\[ \En_{S_i}\brk*{\En_{(x, y)}[\ls_\text{log}(h_i(x), y)]} \leq \inf_{f \in \cF} \En_{(x, y)}[\ls_\text{log}(f(x), y)] + \frac{R(n/M)}{n/M}.\]
For any $p \in \Delta_K$, if $\tilde{p} = \smooth(p)$, then for any $y \in [K]$ we have $-\log(\tilde{p}_y) + \log(p_y) = \log(\frac{p_y}{(1-\mu)p_y + \mu/K}) \leq 2\mu$. So for every $i \in [M]$, we have
\[ \En_{S_i}\brk*{\En_{(x, y)}[\ls_\text{log}(\tilde{h}_i(x), y)]} \leq \En_{S_i}\brk*{\En_{(x, y)}[\ls_\text{log}(h_i(x), y)]} + 2\mu.\]
Putting the above two bounds together, using the specified value of $\mu$ and an application of Markov's inequality, with probability at least $1 - e^{-M} = 1 - \frac{\delta}{2}$, there exists some $i^\star \in [M]$ such that
\begin{equation} \label{eq:markov}
	\En_{(x, y)}[\ls_\text{log}(\tilde{h}_{i^\star}(x), y)] \leq \inf_{f \in \cF} \En_{(x, y)}[\ls_\text{log}(f(x), y)] + \frac{2eR(n/M)}{n/M}.
\end{equation}

The EWOO algorithm in step 3 of the procedure enjoys a regret bound of $O(M\log(n))$ (the online convex	optimization problem is an instance of online portfolio selection over $M$ instruments, see \citep{hazan2007logarithmic}).  Furthermore, the application of $\smooth$ makes the range for the log loss be bounded by $\log(K/\mu)$. Thus, by Corollary 2 of \cite{mehta2016fast}, with probability at least $1-\frac{\delta}{2}$,
\begin{align}
\En_{(x, y)}[\ls_{\text{log}}(g(x), y)] &= \En_{(x, y)}[-\log(\En_{i \sim \bar{q}}[(\tilde{h}_i(x))_y])] \notag\\ 
&\leq \En_{(x, y)}[-\log((\tilde{h}_{i^\star}(x))_y)] + O\prn*{\frac{M\log(n) + \log(K/\mu)\log(\log(n)/\delta)}{n}} \label{eq:ewoo-bound}
\end{align}
Note that $\ls_\text{log}(\tilde{h}_{i^\star}(x), y) = -\log((\tilde{h}_{i^\star}(x))_y)$. Applying the union bound and combining inequalities \eqref{eq:markov} and \eqref{eq:ewoo-bound} with some simplification of the bounds using the value of $M$, with probability at least $1-\delta$ we have
\begin{align*}
\En_{(x, y)}[\ls_{\text{log}}(g(x), y)] &\leq \inf_{f \in \cF} \En_{(x, y)}[\ls_\text{log}(f(x), y)] + O\prn*{\frac{\log\prn*{\frac{1}{\delta}}R\prn*{\frac{n}{\log(1/\delta)}} + \log(Kn)\log\prn*{\frac{\log(n)}{\delta}}}{n}}.
\end{align*}
\end{proof}


\subsection{Details from \pref{sec:general_class}}
\label{app:general_class}

For this section we let $\ls$ denote the unweighted multiclass logistic loss: the multiclass logistic loss defined in \pref{sec:prelims} for the special case where $\cY=\crl*{e_i}_{i\in\brk{K}}$.
Before proving \pref{thm:logistic_minimax} we need a few preliminaries. First, we state a version of the Aggregating Algorithm with the logistic loss for finite classes.
\begin{lemma}
  \label{lem:aggregating_finite}
  Let $\cF$ be any finite class of sequences of the form $f=(f_{t})_{t\leq{}n}$ with $f_t\in\bbR^{K}$, where each $f_t$ is available at time $t$ and may depend on $y_{1:t-1}$. Define a strategy
  \begin{enumerate}
  \item $P_t(f)\propto\exp\prn*{-\sum_{s=1}^{t-1}\ls(f_s, y_s)}$ (so $P_{1}=\mathrm{Uniform}(\cF)$).
  \item $\hat{z}_t = \mb{\sigma}^{+}(\mathrm{smooth}_{\frac{1}{n}}(\En_{f\sim{}P_t}\brk*{\mb{\sigma}(f_t)}))$.
  \end{enumerate}
  This strategy enjoys a regret bound of
  \begin{equation}
    \label{eq:logistic_finite}
    \sum_{t=1}^{n}\ls(\hat{z}_t, y_t) - \min_{f\in\cF}\sum_{t=1}^{n}\ls(f_t, y_t) \leq{} \log\abs*{\cF} + 2.
  \end{equation}
  Furthermore, the predictions satisfy $\nrm*{\hat{z}_t}_{\infty}\leq{}\log(Kn)$.
\end{lemma}
\begin{proof}[\pfref{lem:aggregating_finite}]
  First consider the closely related strategy $\wt{z}_t \ldef \mb{\sigma}^{+}(\En_{f\sim{}P_t}\brk*{\mb{\sigma}(f(x_t))})$. In light of the $1$-mixability for the logistic loss proven in \pref{prop:unweighted-mixability}, $\wt{z}_t$ is precisely the finite class version of the Aggregating Algorithm, which guarantees \citep{PLG}:
  \[
    \sum_{t=1}^{n}\ls(\wt{z}_t, y_t) - \min_{f\in\cF}\sum_{t=1}^{n}\ls(f_t, y_t) \leq{} \log\abs*{\cF}.
  \]
  To establish the final result we simply appeal to \pref{lem:logistic_bounded}, using that $\mb{\sigma}(\mb{\sigma}^{+}(p)) = p\;\forall{}p\in\Delta_{K}$.
  \end{proof}


We now formally define a multiclass generalization of a sequential cover.

\begin{definition}
For any set $\cZ$, a $\mathcal{Z}$-valued $K$-ary tree of depth $n$ is a sequence $\mathbf{z} = (\mathbf{z}_1,\ldots,\mathbf{z}_n)$ of $n$ mappings with $\mathbf{z}_t: [K]^{t-1} \to \mathcal{Z}$. 
\end{definition}
\begin{definition}\label{def:cover}
A set $V$  of $\mathbb{R}^K$-valued $K$-ary trees is an $\alpha$-cover (w.r.t. the $L_p$ norm) of $\F$ on an $\X$-valued $K$-ary tree $\x$ of depth $n$ with loss $\ls$ if
$$
\forall f \in \F, ~y \in [K]^n,  ~\exists \vv \in V ~\textnormal{s.t.}~ \left(\frac{1}{n} \sum_{t=1}^n \max_{y'_t\in\brk*{K}}\left|\ell(f(\x_t(y)),y'_t) - \ell(\vv_t(y),y'_t) \right|^p\right)^{1/p} \le \alpha.
$$ 
\end{definition}
\begin{definition}
The $L_p$ covering number of $\F$ on tree $\x$ is defined as
$$
\mathcal{N}_p(\alpha,\ell \circ \F,\x) \ldef \min\{|V| : V \textrm{ is an }\alpha\textnormal{-cover of $\cF$ on $\x$ w.r.t. the $L_p$ norm} \}.
$$
Further, define $\mathcal{N}_p(\alpha,\ell \circ \F) = \sup_\x \mathcal{N}_p(\alpha,\ell \circ \F,\x)$. 
\end{definition}
  
We also need a slight generalization of the notion of covering number defined in \pref{def:cover} for intermediate results.
  
  \begin{definition}\label{def:cover_general}
  Let $U$ be a collection of $\bbR^{K}$-valued $K$-ary trees. A set $V$  of $\mathbb{R}^K$-valued $K$-ary trees is an $\alpha$-cover with respect to the $L_p$ norm for $U$ if
\[
\forall \uu\in{}U, ~y \in [K]^n,  ~\exists \vv \in V ~\textnormal{s.t.}~ \left(\frac{1}{n} \sum_{t=1}^n \max_{y'_t\in\brk*{K}}\left|\ell(\uu_t(y),y'_t) - \ell(\vv_t(y),y'_t) \right|^p\right)^{1/p} \le \alpha.
\]
\end{definition}
\begin{definition}
\label{def:covering_number_general}
The $L_p$ covering number for a collection of trees $U$ with loss $\ls$ is
\[
\mathcal{N}_p(\alpha, \ls\circ{}U) \ldef \min\{|V| : V \textrm{ is an }\alpha\textnormal{-cover of $U$ w.r.t. the } L_p\textnormal{ norm} \}.
\]
\end{definition}

  \begin{proof}[\pfref{thm:logistic_minimax}]
    Define a subset of the output space:
  \[
    \cZ \ldef{} \crl*{z\in\bbR^{K}\mid{} \nrm*{z}_{\infty} \leq{} \log(Kn)}.
  \]
  
We move to an upper bound on the minimax value by restricting predictions to $\cZ$:
\begin{align*}
  \mathcal{V}_n(\F) & = \dtri*{\sup_{x_t\in\cX} \inf_{\hat{z}_t\in\bbR^{K}} \max_{y_t \in [K]}}_{t=1}^n\left[ \sum_{t=1}^n \ell(\hat{z}_t,y_t) - \inf_{f \in \mathcal{F}} \sum_{t=1}^n \ell(f(x_t),y_t)\right]\\
                    & \leq{} \dtri*{\sup_{x_t\in\cX} \inf_{\hat{z}_t\in\cZ} \max_{y_t \in [K]}}_{t=1}^n\left[ \sum_{t=1}^n \ell(\hat{z}_t,y_t) - \inf_{f \in \mathcal{F}} \sum_{t=1}^n \ell(f(x_t),y_t)\right].
\end{align*}
Note that $\cZ$ is a compact subset of a separable metric space and that $\ls$ is convex with respect to $\hat{z}$. Therefore, using repeated application of minimax theorem following \cite{RakSriTew10}\footnote{See \cite{RakSriTew10} for an extensive discussion of the technicalities.} the minimax value can be written as:
\begin{align}
  &= \dtri*{\sup_{x_t\in\cX} \sup_{p_t \in \Delta_{K}} \inf_{\hat{z}_t\in\cZ} \mathbb{E}_{y_t \sim p_t}}_{t=1}^n\left[ \sum_{t=1}^n \ell(\hat{z}_t,y_t) - \inf_{f \in \mathcal{F}} \sum_{t=1}^n \ell(f(x_t),y_t)\right].\notag
    \intertext{Now we perform a standard manipulation of the $\sup$ and loss terms as in \cite{RakSriTew10}:}
&= \dtri*{\sup_{x_t\in\cX} \sup_{p_t \in \Delta_{K}}  \mathbb{E}_{y_t \sim p_t}}_{t=1}^n\left[ \sum_{t=1}^n \inf_{\hat{z}_t\in\cZ} \mathbb{E}_{y_t \sim p_t}\left[\ell(\hat{z}_t,y_t)\right] - \inf_{f \in \mathcal{F}} \sum_{t=1}^n \ell(f(x_t),y_t)\right]\label{eq:inf_inside}\\
&= \sup_{\x, \p} \mathbb{E}_{y \sim \p}\left[ \sum_{t=1}^n \inf_{\hat{z}_t\in\cZ} \mathbb{E}_{y_t \sim \p_t(y)}\left[\ell(\hat{z}_t,y_t)\right] - \inf_{f \in \mathcal{F}} \sum_{t=1}^n \ell(f(\x_t(y)),y_t)\right]\label{eq:logistic_tree}.
\end{align}
In the final line above we have introduced new notation. $\x$ and $\p$ are $\X$- and $\Delta_K$-valued $K$-ary trees of depth $n$. That is, $\x = (\x_1,\ldots,\x_n)$ where $\x_t : [K]^{t-1} \to \X$ and similarly for the tree $\p=(\p_1,\ldots,\p_n)$,  $\p_{t}:\brk*{K}^{t-1}\to\Delta_{K}$. The notation ``$y \sim \p$'' refers to the process in which we first draw $y_1 \sim \p_1$, then draw $y_t \sim \p_t(y_1,\ldots,y_{t-1})$ for subsequent timesteps $t$. We also overload the notation as $\p_{t}(y)\ldef\p_{t}(y_{1:t-1})$, and likewise for $\x$.

With this notation, \pref{eq:logistic_tree} is seen to be \pref{eq:inf_inside} rewritten using that at time $t$, based on draw of previous $y$s, $x_t$ and $p_t$ are chosen to maximize the remaining game value; this process be represented via $K$-ary tree.

Note that the sequence $(\hat{z}_{t})_{t\leq{}n}$ being minimized over in \pref{eq:inf_inside} can depend on the full trees $\x$ and $\p$, but that it is adapted to the path $(y_t)_{t\leq{}n}$, meaning that the value at time $t$ ($\hat{z}_{t}$) can only depend on the $\yr[t-1]$. This property is imporant because the choice we exhibit for $(\hat{z}_{t})_{t\leq{}n}$ will indeed depend on the full trees.

In light of the discussion in \pref{sec:general_class}, the key advantage of having moved to the dual game above is that we can condition on the $K$-ary tree $\x$ and cover $\F$ only on this tree. Let $V^\gamma$ be a minimal $\gamma$-sequential cover of $\ell \circ \mathcal{F}$ on the tree $\mathbf{x}$ with respect to the $L_2$ norm (in the sense of \pref{def:cover}). 

Keeping the tree $\x$ fixed, for each tree $\vv \in V^\gamma$, each $f \in \F$, we define a class of trees $\cF_{\vv}$ ``centered'' at $\vv$---in a sense that will be made precise in a moment---via the following procedure.

\begin{itemize}[leftmargin=*]
\item $\cF_{\vv} = \emptyset$.
\item For each $f\in\cF$ and $y\in\brk*{K}^{n}$ with $\sqrt{\frac{1}{n} \sum_{t=1}^n \max_{y''_t\in\brk*{K}}(\ell(f(\x_t(y)),y''_t) - \ell(\vv_t(y),y''_t) )^2} \le \gamma$: 
\begin{itemize}
\item Define a $\bbR^{K}$-valued $K$-ary tree $\uu_{f,y}$ via: For each $y'\in\brk*{K}^{n}$, 
\[
(\uu_{f,y})_{t}(y') \ldef f(\x_{t}(y'))\ind\crl*{y'_1=y_1,\ldots,y'_{t-1}=y_{t-1}} + \vv_{t}(y')\ind\crl*{\neg{}(y'_1=y_1,\ldots,y'_{t-1}=y_{t-1})}.
\]
In other words, $\uu_{f,y}$ is equal to $f\circ\x$ on the path $y$, and equal to $\vv$ everywhere else.
\item Add $\uu_{f,y}$ to $\cF_{\vv}$.
\end{itemize}
\end{itemize}

The class $\cF_{\vv}$ has two important properties which are formally proven in an auxiliary lemma, \pref{lem:fv_properties}: First, its $L_2$ covering number is (up to low order terms) bounded in terms of the $L_2$ covering number of the class $\cF\circ{}\x$, so it has similar complexity to this class. Second, its $L_2$ radius is bounded by $\gamma$, in the sense that its covering number at scale $\gamma$ is at most $1$.

Note that on any path $y \in [K]^n$ and for each $f \in \F$, there exist $\vv \in V^\gamma$ and $\uu \in \F_\vv$ such that $f(\x_t(y)) = \uu_{t}(y)$. This is because a $\vv$ that is $\gamma$-close to $f$ on the path $y$ through $\x$ is guaranteed by the cover property of $V^{\gamma}$, and so we can take $\uu_{f,y}$ in $\cF_{\vv}$ as the desired $\uu$. This implies that
$$
\inf_{f \in \F }\sum_{t=1}^n \ell(f(\x_t(y)),y_t) \ge \min_{\vv \in V^\gamma} \inf_{\uu \in \F_\vv }\sum_{t=1}^n \ell(\uu_t(y),y_t).
$$

With this we are ready to return to the minimax rate. We already established that
\begin{align}
 \mathcal{V}_n(\F)  &\leq \sup_{\x, \p} \mathbb{E}_{y \sim \p}\left[ \sum_{t=1}^n \inf_{\hat{z}_t\in\cZ} \mathbb{E}_{y_t \sim \p_t(y)}\left[\ell(\hat{z}_t,y_t)\right] - \inf_{f \in \mathcal{F}} \sum_{t=1}^n \ell(f(\x_t(y)),y_t)\right]. \notag 
 \intertext{We now move to an upper bound based on the constructions for the tree collections $V^{\gamma}$ and $\crl*{\cF_{\vv}}_{\vv\in{}V^{\gamma}}$. These collections depend only on the tree $\x$ at the outer supremum above. Writing the choice of these collections as an infimum to make its dependence on the other quantities in the random process as explicit as possible, and using the containment just shown:}
  & \le \sup_{\x}\inf_{V^{\gamma}}\inf_{\crl*{\cF_{\vv}}_{\vv\in{}V^{\gamma}}}\sup_{\p}\mathbb{E}_{y \sim \p}\left[ \sum_{t=1}^n \inf_{\hat{z}_t\in\cZ} \mathbb{E}_{y_t \sim \p_t(y)}\left[\ell(\hat{z}_t,y_t)\right] - \min_{\vv \in V^\gamma} \inf_{\uu \in \mathcal{F}_\vv} \sum_{t=1}^n \ell(\uu_t(y),y_t)\right]. \notag
\end{align}
For the last time in the proof, we introduce a new collection of trees. For each $\vv\in{}V^{\gamma}$ we introduce a $\cZ$-valued $K$-ary tree $\yhtree^{\vv}$, with $\yhtree^{\vv}_{t}:\brk*{K}^{t-1}\to\cZ$. We postpone explicitly constructing the trees for now, but the reader may think of each tree $\yhtree^{\vv}$ as representing the optimal strategy for the set $\cF_{\vv}$ in a sense that will be made precise in a moment.
{\small
\begin{align}
  & \begin{aligned}= \sup_{\x}\inf_{V^{\gamma}}\inf_{\crl*{\cF_{\vv}}_{\vv\in{}V^{\gamma}}}\inf_{\crl*{\yhtree^{\vv}}_{\vv\in{}V^{\gamma}}}\sup_{\p} \mathbb{E}_{y \sim \p}\biggl[& \sum_{t=1}^n \inf_{\hat{z}_t\in\cZ} \mathbb{E}_{y_t \sim \p_t(y)}\left[\ell(\hat{z}_t,y_t)\right] \\
    &- \min_{\vv \in V^\gamma}\left\{ \sum_{t=1}^n \ell(\yhtree_t^\vv(y),y_t) - \sum_{t=1}^n \ell(\yhtree_t^\vv(y),y_t) +  \inf_{\uu \in \mathcal{F}_\vv} \sum_{t=1}^n \ell(\uu_t(y),y_t)\right\}\biggr]\end{aligned} \notag \\
& \leq\sup_{\x}\inf_{V^{\gamma}}\inf_{\crl*{\cF_{\vv}}_{\vv\in{}V^{\gamma}}}\inf_{\crl*{\yhtree^{\vv}}_{\vv\in{}V^{\gamma}}}\left\{\begin{aligned}~& \underbrace{\sup_{\p} \mathbb{E}_{y \sim \p}\biggl[ \sum_{t=1}^n \inf_{\hat{z}_t\in\cZ} \mathbb{E}_{y_t \sim \p_t(y)}\left[\ell(\hat{z}_t,y_t)\right] - \min_{\vv \in V^\gamma}\sum_{t=1}^n \ell(\yhtree_t^\vv(y),y_t)\biggr]}_{(\star)} \\
  & + \underbrace{\sup_{\p}\mathbb{E}_{y \sim \p}\left[ \max_{\vv \in V^\gamma}\left\{ \sum_{t=1}^n \ell(\yhtree_t^\vv(y),y_t) -  \inf_{\uu \in \mathcal{F}_\vv} \sum_{t=1}^n \ell(\uu_t(y),y_t)\right\}\right]}_{(\star\star)}
  \end{aligned}\right\}.\label{eq:interval}
\end{align}}

We now bound terms $(\star)$ and $(\star\star)$ individually by instantiating specific choices for $(\hat{z}_t)_{t\leq{}n}$ and $\crl*{\yhtree^{\vv}}$.

\paragraph{Term $(\star)$}
We select $(\hat{z}_t)_{t\leq{}n}$ using the Aggregating Algorithm as configured in \pref{lem:aggregating_finite}, taking $\cF$ to be the finite collection of sequences $\crl*{\yhtree^{\vv}}_{\vv\in{}V^{\gamma}}$. Since each tree has the property that $\yhtree^{\vv}_{t}$ only depends on $y_{1:t-1}$, \pref{lem:aggregating_finite} indeed applies, which means that for any sequence $\yr[n]\in\brk*{K}^{n}$ of labels the algorithm deterministically satisfies the regret inequality
\[
\sum_{t=1}^n \ell(\hat{z}_t,y_t) - \min_{\vv \in V^\gamma}\sum_{t=1}^n \ell(\yhtree_t^\vv(y),y_t) \leq{} \log\abs*{V^{\gamma}} + 2.
\]
Since the algorithm guarantees $\nrm*{\hat{z}_{t}}_{\infty}\leq{}\log(Kn)$, one can verify that $\hat{z}_{t}\in\cZ$. Furthermore, $\hat{z}_{t}$ depends only on $\yr[t-1]$, and so the predictions of the Aggregating Algorithm are a valid choice for the infimum in $(\star)$. This implies that
\[
(\star) \leq{} \sup_{\x}\log\abs*{V^{\gamma}} + 2 \leq \log\cN_{2}(\gamma, \ls\circ\cF) + 2,
\]
since the regret inequality holds for every possible draw of $\yr[n]$ in the expression $(\star)$.

\paragraph{Term $(\star\star)$}

First, observe that each tree class $\cF_{\vv}$ is uniformly bounded in the sense that \[\sup_{\uu\in\cF_{\vv}}\sup_{y\in\brk*{K}^{n}}\max_{t\in\brk*{n}}\nrm*{\uu_{t}(y)}_{\infty}<\infty.\] This holds because $\uu_{t}(y)$ is either equal to $\vv_{t}(y)$, which is finite, or is equal to $f(\x_{t}(y))$ for some $f\in\cF$, and the class $\cF$ was already assumed to be uniformly bounded.

To bound this term we need a variant of the sequential Rademacher complexity regret bound of \citep{RakSriTew10}, which shows that there exists a deterministic strategy for competing against any collection of trees. This is proven in the auxiliary \pref{lem:rademacher_strategy} following this proof.

In particular, for each tree class $\cF_{\vv}$, there exists a deterministic strategy $\hat{y}_{t}^{\vv}$ that guarantees the inequality
\[
  \sum_{t=1}^n \ell(\hat{y}^{\vv}_{t},y_t) -  \inf_{\uu \in \mathcal{F}_\vv} \sum_{t=1}^n \ell(\uu_t(y),y_t) \leq
  2\cdot\max_{\y,\y'}\Enn_{\eps}\sup_{\uu\in{}\cF_{\vv}}\left[ \sum_{t=1}^n \eps_{t}\ell(\uu_{t}(\y_{1:t-1}(\eps)),\y'_t(\eps))\right] + 2,
 \]
 holds for every sequence, where the supremum on the right-hand-side ranges over $\brk*{K}$-valued binary trees. Futhermore, $\yh_{t}^{\vv}$ is guaranteed by \pref{lem:rademacher_strategy} to lie in the class $\cZ$. We choose this strategy for the collection $\crl*{\yhtree^{\vv}}$ being minimized over in \pref{eq:interval}. Since the regret inequality from \pref{lem:rademacher_strategy} holds deterministically for all sequences $y$ for each $\vv$, we have that
 \[
   (\star\star) \leq{} 2\cdot\max_{\vv\in{}V^{\gamma}}\max_{\y,\y'}\Enn_{\eps}\sup_{\uu\in{}\cF_{\vv}}\left[ \sum_{t=1}^n \eps_{t}\ell(\uu_{t}(\y_{1:t-1}(\eps)),\y'_t(\eps))\right] + 2.
 \]
 For each choice of $\vv$, $\y$, $\y'$ at the outer supremum, we define a class of real-valued trees $W_{\vv, \y, \y'}$  via  $\crl*{(\ww_{t})_{t\leq{}n}\;:\;\ww_{t}(\eps) \ldef \ls(\uu_{t}(\y(\eps_{1:t-1})), \y'_{t}(\eps)) \mid{} \uu\in\cF_{\vv}}$. \pref{lem:chaining_trees} then implies
 \[
   (\star\star) \leq{} 2\max_{\vv\in{}V^{\gamma}}\max_{\y,\y'}\inf_{\alpha>0}\crl*{
      4\alpha{}n + 12\int_{\alpha}^{\mathrm{rad}_{2}(W_{\vv,\y,\y'})}\sqrt{n\log\cN_{2}(\delta, W_{\vv,\y,\y'})}d\delta
    } + 2,
 \]
 with the real-valued covering number $\cN_{2}$ and radius $\mathrm{rad}_{2}$ defined as in \pref{lem:chaining_trees}.

 We now show how to bound this covering number in terms of the covering number for $\cF_{\vv}$. Suppose that $Z$ is a collection of $\bbR^{K}$-valued $K$-ary trees that form a $\delta$-cover for $\cF_{\vv}$ in the sense of \pref{def:cover_general}. Then we have
 \begin{align*}
   &\sup_{\uu\in\cF_{\vv}}\max_{\eps\in\pmo^{n}}\inf_{\zz\in{}Z}\sqrt{\frac{1}{n}\sum_{t=1}^{n}\prn*{\ls(\uu_t(\y(\eps)),\y'_t(\eps)) - \ls(\zz_{t}(\y(\eps)),\y'_t(\eps))}^{2}} \\
   &\leq{} \sup_{\uu\in\cF_{\vv}}\max_{\eps\in\pmo^{n}}\inf_{\zz\in{}Z}\sqrt{\frac{1}{n}\sum_{t=1}^{n}\max_{y'_{t}\in\brk{K}}\prn*{\ls(\uu_t(\y(\eps)),y'_t) - \ls(\zz_t(\y (\eps)),y'_t)}^{2}} \\
   &\leq{} \sup_{\uu\in\cF_{\vv}}\max_{y\in\brk*{K}^{n}}\inf_{\zz\in{}Z}\sqrt{\frac{1}{n}\sum_{t=1}^{n}\max_{y'_{t}\in\brk{K}}\prn*{\ls(\uu_t(y),y'_t) - \ls(\zz_t(y),y'_t)}^{2}}\\
   &\leq{} \delta.
 \end{align*}
 This implies that for any cover of $\cF_{\vv}$ in the sense of \pref{def:cover_general} we can construct a cover for $W_{\vv,\y,\y'}$ at the same scale using the construction $\crl*{(\ww_{t})_{t\leq{}n}\;:\;\ww_{t}(\eps) \ldef \ls(\zz_{t}(\y(\eps_{1:t-1})), \y'_{t}(\eps)) \mid{} \zz\in{}Z}$. Consequently, we have
  \[
    (\star\star) \leq{} 2\max_{\vv\in{}V^{\gamma}}\inf_{\alpha>0}\crl*{
      4\alpha{}n + 12\int_{\alpha}^{\mathrm{rad}_{2}(\cF_{\vv})}\sqrt{n\log\cN_{2}(\delta, \ls\circ\cF_{\vv})}d\delta
    } + 2.
  \]
  In light of \pref{lem:fv_properties}, this is further upper bounded by
  \begin{align*}
    (\star\star) &\leq{} 2\inf_{\alpha>0}\crl*{
      4\alpha{}n + 12\int_{\alpha}^{\gamma}\sqrt{n\log\prn*{\cN_{2}(\delta, \ls\circ\cF, \x)n}}d\delta
                   } + 2 \\
    &\leq{} 2\inf_{\alpha>0}\crl*{
      4\alpha{}n + 12\int_{\alpha}^{\gamma}\sqrt{n\log\prn*{\cN_{2}(\delta, \ls\circ\cF)n}}d\delta
    } + 2.    
  \end{align*}

\paragraph{Final bound}
Combining $(\star)$ and $(\star\star)$, we have
\[
  \mathcal{V}_n(\F)
  \leq{} \log\cN_{2}(\gamma, \ls\circ\cF) + \inf_{\gamma\geq{}\alpha>0}\crl*{
      8\alpha{}n + 24\int_{\alpha}^{\gamma}\sqrt{n\log\prn*{\cN_{2}(\delta, \ls\circ\cF)n}}d\delta
    } + 4.
\]
for any fixed $\gamma$. Optimizing over $\gamma$ yields the result.
\end{proof}

\begin{lemma}
\label{lem:fv_properties}
Let $\cF_{\vv}$ be defined as in the proof of \pref{thm:logistic_minimax} for trees $\vv$ and $\x$ and scale $\gamma$. Then it holds that
\begin{enumerate}
\item $\mathcal{N}_2(\gamma,\ell \circ \cF_{\vv} ) \leq{} 1$.
\item $\mathcal{N}_2(\alpha,\ell \circ \cF_{\vv} ) \leq{} n\cdot{}\mathcal{N}_2(\alpha,\ell \circ \cF, \x )$ for all $\alpha>0$.
\end{enumerate}

\end{lemma}
\begin{proof}[\pfref{lem:fv_properties}]~\\
\textbf{First claim}~~~~
This is essentially by construction. Recall that each element of $\cF_v$ is of the form
\[
(\uu_{f,y})_{t}(y') \ldef f(\x_{t}(y'))\ind\crl*{y'_1=y_1,\ldots,y'_{t-1}=y_{t-1}} + \vv_{t}(y')\ind\crl*{\neg{}(y'_1=y_1,\ldots,y'_{t-1}=y_{t-1})}.
\]
for some path $y\in\brk*{K}^n$ and $f\in\cF$ for which
\begin{equation}
\label{eq:radius_gamma}
\sqrt{\frac{1}{n} \sum_{t=1}^n \max_{y''_t\in\brk*{K}}(\ell(f(\x_t(y)),y''_t) - \ell(\vv_t(y),y''_t) )^2} \le \gamma.
\end{equation}
These properties imply that $\crl*{\vv}$ is a sequential $\gamma$-cover. Indeed, using the explicit form for $\mb{u}_{f,y}$ above, it can be seen that for each path $y'\in\brk*{K}^{n}$, there exists some time $1<\tau\leq{}n+1$ such that
\[
(\uu_{f,y})_{t}(y') = \left\{
\begin{array}{ll}
f(\x_{t}(y')), & \textrm{ if } t < \tau,\\
\vv_{t}(y'), & \textrm{ if } t \geq{} \tau.
\end{array}
\right.
\]
It also holds that $y_{t}=y'_{t}$ for all $t<\tau-1$.

Using this representation we have that for any path $y'\in\brk*{K}^{n}$:
\begin{align*}
&\sqrt{\frac{1}{n} \sum_{t=1}^n \max_{y''_t\in\brk*{K}}(\ell((\uu_{f,y})_{t}(y'),y''_t) - \ell(\vv_t(y'),y''_t) )^2}\\
&= \sqrt{\frac{1}{n} \sum_{t=1}^{\tau-1} \max_{y''_t\in\brk*{K}}(\ell(f(\x_{t}(y'),y''_t) - \ell(\vv_t(y'),y''_t) )^2}.
\end{align*}
Now use that $\x_{1},\ldots,\x_{\tau-1}$ and $\vv_1,\ldots,\vv_{\tau-1}$ only depend on $y'_1,\ldots,y'_{\tau-2}$, and that $y'_1,\ldots,y'_{\tau-2} = y_1,\ldots,y_{\tau-2}$:
\begin{align*}
&= \sqrt{\frac{1}{n} \sum_{t=1}^{\tau-1} \max_{y''_t\in\brk*{K}}(\ell(f(\x_{t}(y),y''_t) - \ell(\vv_t(y),y''_t) )^2} \\
&\leq \sqrt{\frac{1}{n} \sum_{t=1}^{n} \max_{y''_t\in\brk*{K}}(\ell(f(\x_{t}(y),y''_t) - \ell(\vv_t(y),y''_t) )^2} \\
&\leq{} \gamma.
\end{align*}

\textbf{Second claim}~~~~
Let $V$ be a cover for $\ls\circ\cF$ on $\x$ of size $\mathcal{N}_2(\alpha,\ell \circ \cF, \x)$. Assume $\abs*{V}<\infty$ as the claim holds trivially otherwise. We will construct from $V$ a cover $\wt{V}$ for $\ls\circ\F_{\vv}$ with the following procedure:
\begin{itemize}
\item $\wt{V}=\emptyset$.
\item For each $K$-ary $\bbR^{K}$-valued tree $\zz\in{}V$ and each time $\tau\in\crl*{2,\ldots,n+1}$:
\begin{itemize}
\item Construct a $K$-ary $\bbR^{K}$-valued tree $\zz^{(\tau)}$ via
\[
\zz^{(\tau)}_{t}(y) = \zz_{t}(y)\ind\crl*{t<\tau} + \vv_{t}(y)\ind\crl*{t\geq{}\tau}.
\]
\item Add $\zz^{(\tau)}$ to $\wt{V}$.
\end{itemize}
\end{itemize}
Clearly $\abs*{\wt{V}}\leq{}n\cdot\abs*{V}$. We now show that $\wt{V}$ is an $\alpha$-cover for $\ls\circ\cF_{\vv}$.

Let $\uu_{f,y}$ be an element of $\cF_{\vv}$ of the form described in the proof of the first claim and let $y'\in\brk*{K}^{n}$ be a particular path. Let $\tau$ be such that $(\uu_{f,y})_{t}(y') = f(\x_{t}(y'))\ind\crl*{t<\tau} + \vv_{t}(y')\ind\crl*{t\geq{}\tau}$. Let $\zz\in{}V$ be $\alpha$-close to $f$ on the path $y'$ through $\x$, i.e.
\[
\sqrt{\frac{1}{n} \sum_{t=1}^{n} \max_{y''_t\in\brk*{K}}(\ell(f(\x_{t}(y'),y''_t) - \ell(\zz_t(y'),y''_t) )^2} \leq{} \alpha.
\]
Existence of such a $\zz$ is guaranteed by the cover property of $V$. We will show that $\zz^{(\tau)}$ is $\alpha$-close to $\uu_{f,y}$ on $y'$. Indeed, we have
\begin{align*}
&\sqrt{\frac{1}{n} \sum_{t=1}^n \max_{y''_t\in\brk*{K}}(\ell((\uu_{f,y})_{t}(y'),y''_t) - \ell(\zz^{(\tau)}_t(y'),y''_t) )^2}\\
&= \sqrt{\frac{1}{n} \sum_{t=1}^{\tau-1} \max_{y''_t\in\brk*{K}}(\ell(f(\x_{t}(y'),y''_t) - \ell(\zz_{t}(y'),y''_t) )^2
+ \frac{1}{n} \sum_{t=\tau}^{n} \max_{y''_t\in\brk*{K}}(\ell(\vv_{t}(y'),y''_t) - \ell(\vv_t(y'),y''_t) )^2} \\
&= \sqrt{\frac{1}{n} \sum_{t=1}^{\tau-1} \max_{y''_t\in\brk*{K}}(\ell(f(\x_{t}(y'),y''_t) - \ell(\zz_{t}(y'),y''_t) )^2} \\
&\leq{} \sqrt{\frac{1}{n} \sum_{t=1}^{n} \max_{y''_t\in\brk*{K}}(\ell(f(\x_{t}(y'),y''_t) - \ell(\zz_{t}(y'),y''_t) )^2} \\
&\leq{} \alpha.
\end{align*}
Since this argument works for any $\uu_{f,y}\in\cF_{\vv}$ this establishes that $\wt{V}$ is an $\alpha$-cover of $\cF_{\vv}$.
\end{proof}

The next lemma is almost the same as the sequential Rademacher complexity bound in \cite{RakSriTew10}, with the only technical difference being that the learner competes with a class of trees rather than a class of fixed functions. It is proven using the same argument as in that paper.
\begin{lemma}
  \label{lem:rademacher_strategy}
  Let $U$ be any collection of $\bbR^{K}$-valued $K$-ary trees of depth $n$. Suppose that $C\ldef\sup_{\uu\in{}U}\sup_{y\in\brk*{K}^{n}}\max_{t\in\brk*{n}}\nrm*{\uu_{t}(y)}_{\infty}<\infty$. Then there exists a strategy $\hat{z}_{t}$ that guarantees
  \[
    \sum_{t=1}^{n}\ls(\zh_t, y_t) - \inf_{\uu\in{}U}\sum_{t=1}^{n}\ls(\uu_{t}(y), y_t) \leq{} 2\cdot\max_{\y,\y'}\Enn_{\eps}\sup_{\uu\in{}U}\left[ \sum_{t=1}^n \eps_{t}\ell(\uu_{t}(\y_{1:t-1}(\eps)),\y'_t(\eps))\right] + 2,
  \]
  where $\y$ and $\y'$ are $\brk*{K}$-valued binary trees of depth $n$ and $\eps=(\eps_1,\ldots,\eps_n)$ are Rademacher random variables.

  Furthermore, the predictions $(\hat{z}_{t})_{t\leq{}n}$ satisfy $\nrm*{\hat{z}_{t}}_{\infty}\leq{}\log(Kn)$.

\end{lemma}
\begin{proof}[\pfref{lem:rademacher_strategy}]
      Define $\cZ \ldef{} \crl*{z\in\bbR^{K}\mid{} \nrm*{z}_{\infty} \leq{} C}$. The minimax optimal regret amongst deterministic strategies taking values in $\cZ$ is given by
\begin{align*}
  \mathcal{V}_n(U) & \ldef \dtri*{\inf_{\hat{z}_t\in\bbR^{K}} \max_{y_t \in [K]}}_{t=1}^n\left[ \sum_{t=1}^n \ell(\hat{z}_t,y_t) - \inf_{\uu\in{}U} \sum_{t=1}^n \ell(\uu_t(y),y_t)\right].
\end{align*}
Once again, this proof closely follows the sequential Rademacher complexity bound from \cite{RakSriTew10}.
We only sketch the first few steps for this proof as they are identical to the first few steps of the proof of \pref{thm:logistic_minimax}, which is admissible due to compactness of $\cZ$. Using the minimax swap as in that theorem, we can move to an upper bound of
\begin{align*}
  &\leq \dtri*{\sup_{p_t \in \Delta_{K}}  \mathbb{E}_{y_t \sim p_t}}_{t=1}^n\left[ \sum_{t=1}^n \inf_{\hat{z}_t\in\cZ} \mathbb{E}_{y_t \sim p_t}\left[\ell(\hat{z}_t,y_t)\right] - \inf_{\uu\in{}U} \sum_{t=1}^n \ell(\uu_{t}(y),y_t)\right] \\
  &= \dtri*{\sup_{p_t \in \Delta_{K}}  \mathbb{E}_{y_t \sim p_t}}_{t=1}^n\sup_{\uu\in{}U}\left[ \sum_{t=1}^n \inf_{\hat{z}_t\in\cZ} \mathbb{E}_{y_t \sim p_t}\left[\ell(\hat{z}_t,y_t)\right] - \sum_{t=1}^n \ell(\uu_{t}(y),y_t)\right].
    \intertext{Now we choose $\hat{z}_{t}$ to match the value of $\uu_{t}(y) = \uu_{t}(\yr[t-1])$, which is possible by definition of $\cZ$:}
  &\leq \dtri*{\sup_{p_t \in \Delta_{K}}  \mathbb{E}_{y_t \sim p_t}}_{t=1}^n\sup_{\uu\in{}U}\left[ \sum_{t=1}^n \mathbb{E}_{y_t \sim p_t}\left[\ell(\uu_{t}(y),y_t)\right] - \sum_{t=1}^n \ell(\uu_{t}(y),y_t)\right].
    \intertext{Using Jensen's inequality, we pull the conditional expectaitons in the first term outside the supremum over $\uu$ by introducing a tangent sequence $(y_t')_{t\leq{}n}$, where $y'_{t}$ follows the distribution $p_{t}$ conditioned on $\yr[t-1]$.}
  &\leq \dtri*{\sup_{p_t \in \Delta_{K}}  \mathbb{E}_{y_t,y'_t \sim p_t}}_{t=1}^n\sup_{\uu\in{}U}\left[ \sum_{t=1}^n \ell(\uu_{t}(y),y'_t) - \sum_{t=1}^n \ell(\uu_{t}(y),y_t)\right].
    \intertext{Since $y_t$ and $y'_t$ are conditionally i.i.d., we can introduce a Rademacher random variable $\eps_{t}$ at each timestep $t$ as follows:}
  &= \dtri*{\sup_{p_t \in \Delta_{K}} \mathbb{E}_{y_t,y'_t \sim p_t}\Enn_{\eps_t}}_{t=1}^n\sup_{\uu\in{}U}\left[ \sum_{t=1}^n \eps_{t}\prn*{\ell(\uu_{t}(y),y'_t) - \ell(\uu_{t}(y),y_t)}\right].
    \intertext{To decouple the arguments to the losses from the arugments to the tree $\uu$, we move to a pessimistic upper bound:}
  &\leq \dtri*{\sup_{p_t \in \Delta_{K}} \mathbb{E}_{y_t\sim p_t}\max_{y'_t,y''_t\in\brk*{K}}\Enn_{\eps_t}}_{t=1}^n\sup_{\uu\in{}U}\left[ \sum_{t=1}^n \eps_{t}\prn*{\ell(\uu_{t}(y),y'_t) - \ell(\uu_{t}(y),y''_t)}\right] \\
  &= \dtri*{\max_{y_t,y'_t,y''_t\in\brk*{K}}\Enn_{\eps_t}}_{t=1}^n\sup_{\uu\in{}U}\left[ \sum_{t=1}^n \eps_{t}\prn*{\ell(\uu_{t}(y),y'_t) - \ell(\uu_{t}(y),y''_t)}\right]. \\
  \intertext{We now complete the symmetrization as follows:}
  &\leq \dtri*{\max_{y_t,y'_t,y''_t\in\brk*{K}}\Enn_{\eps_t}}_{t=1}^n\sup_{\uu\in{}U}\left[ \sum_{t=1}^n \eps_{t}\ell(\uu_{t}(y),y'_t)\right]
    + \dtri*{\max_{y_t,y'_t,y''_t\in\brk*{K}}\Enn_{\eps_t}}_{t=1}^n\sup_{\uu\in{}U}\left[ \sum_{t=1}^n \eps_{t}\ell(\uu_{t}(y),y''_t)\right] \\
  &= 2\cdot\dtri*{\max_{y_t,y'_t\in\brk*{K}}\Enn_{\eps_t}}_{t=1}^n\sup_{\uu\in{}U}\left[ \sum_{t=1}^n \eps_{t}\ell(\uu_{t}(y),y'_t)\right]\\
  &= 2\cdot\max_{\y,\y'}\Enn_{\eps}\sup_{\uu\in{}U}\left[ \sum_{t=1}^n \eps_{t}\ell(\uu_{t}(\y_{1:t-1}(\eps)),\y'_t(\eps))\right].
\end{align*}
In the last line $\y$ and $\y'$ are taken to be $\brk*{K}$-valued binary trees of depth $n$, so that $\y_{t}(\eps) = \y_{t}(\eps_{1},\ldots\eps_{t-1})$ and likewise for $\y'$.

Finally, to guarantee the boundedness of predictions claimed in the lemma statement, we apply \pref{lem:logistic_bounded} to the minimax optimal strategy, for which we just showed regret is bounded by the sequential Rademacher complexity.
\end{proof}

The last auxiliary lemma in this section is a slight variant of the Dudley entropy integral bound for sequential Rademacher complexity. This lemma can be extracted from the proof of Theorem 4 in \cite{rakhlin2015ptrf}. We do not repeat the proof here.
\begin{lemma}
  \label{lem:chaining_trees}
  Let $W$ be a collection of $\bbR$-valued binary trees. Define $\cN_{p}(\alpha, W)$ to be the size of the smallest class of trees $V$ such that
  \begin{equation}
    \label{eq:cover_real}
    \forall \ww\in{}W, \eps\in\pmo^{n},  ~\exists \vv \in V ~\textnormal{s.t.}~ \left(\frac{1}{n} \sum_{t=1}^n\prn*{\ww_{t}(\eps) - \vv_{t}(\eps)}^p\right)^{1/p} \le \alpha.
  \end{equation}
  Let $\mathrm{rad}_{p}(W)\ldef{}\min\crl*{\alpha\mid{}\cN_{p}(\alpha, W) = 1}$. Then it holds that
  \begin{equation}
    \label{eq:chaining_real}
    \Enn_{\eps}\sup_{\ww\in{}W}\sum_{t=1}^{n}\eps_{t}\ww_{t}(\eps) \leq{} \inf_{\alpha>0}\crl*{
      4\alpha{}n + 12\int_{\alpha}^{\mathrm{rad}_{2}(W)}\sqrt{n\log\cN_{2}(\delta, W)}d\delta
    }.
    \end{equation}

\end{lemma}

\subsection{Details from \pref{sec:log_loss}}
\label{app:logloss}

We first define a suitable notion of sequential cover for the log loss setting:
\begin{definition}
  \label{def:cover_real}
  For a fixed $\cX$-valued binary tree $\x$, define $\cN_{\infty}(\alpha, \cF, \x)$ to be the size of the smallest set of $\brk*{0,1}$-valued binary trees $V$ such that
  \[
    \forall f \in \F, ~\eps \in \pmo^{n},  ~\exists \vv \in V ~\textnormal{s.t.}~ \max_{t\in\brk*{n}}\abs*{f(\x_{t}(\eps)) - \vv_{t}(\eps)} \leq{} \alpha.
  \]Further, define $\mathcal{N}_{\infty}(\alpha,\F) = \sup_\x \mathcal{N}_{\infty}(\alpha,\F,\x)$. 
\end{definition}
We also require a generalization of \pref{def:cover_real} for general tree classes.
\begin{definition}
  \label{def:cover_real_tree}
  For a class of $\brk*{0,1}$-valued binary trees $U$, define $\cN_{\infty}(\alpha, U)$ to be the size of the smallest set of $\brk*{0,1}$-valued binary trees $V$ such that
  \[
    \forall \uu \in U, ~\eps \in \pmo^{n},  ~\exists \vv \in V ~\textnormal{s.t.}~ \max_{t\in\brk*{n}}\abs*{\uu_{t}(\eps) - \vv_{t}(\eps)} \leq{} \alpha.
  \]
\end{definition}

We now turn to the proof of \pref{thm:logloss_minimax}. It follows the same structure as the proof in \pref{app:general_class} with a few technical differences related the slightly different notion of cover used and the non-Lipschitzness of the log loss. We first give one more definition.

\begin{definition}
  For any $\delta\in(0,1/2]$, we define the truncation to the range $\brk*{\delta, 1-\delta}$ via $\clip(p) = \max\crl*{\delta, \min\crl*{1-\delta, p}}$.
\end{definition}

The following proposition is a simple consequence of the fact that $\clip$ is $1$-Lipschitz.
\begin{proposition}
  \label{prop:clip_covering}
  For any class of trees $U$ and any $\delta\in(0,1/2]$, $\cN_{\infty}(\alpha, \clip\circ{}U)\leq{}\cN_{\infty}(\alpha, U)$.
\end{proposition}

\begin{proof}[\pfref{thm:logloss_minimax}]
  The proof is very similar to that of \pref{thm:logistic_minimax}. When it would otherwise be repetitive we will only sketch details and instead refer back to the proof of that theorem.

  To begin, fix $\delta\in(0,1/2]$. We will work with the clipped class $\cF^{\delta} = \clip\circ{}\cF$ just as in \cite{PLG}. It was shown there that
  \[
    \cV_{n}^{\log}(\cF) \leq{} \cV_{n}^{\log}(\cF^{\delta}) + \delta{}n.
  \]

  With this restriction, we proceed exactly as in the proof of \pref{thm:logistic_minimax}. First, restrict the learner's predictions to $\brk*{\delta,1-\delta}$ to guarantee boundedness of the loss:
  \begin{align*}
  \mathcal{V}^{\mathrm{log}}_n(\F^{\delta}) & = \dtri*{\sup_{x_t\in\cX} \inf_{\lpred_t\in\brk*{0,1}} \max_{y_t \in \crl*{0,1}}}_{t=1}^n\left[ \sum_{t=1}^n \logloss(\lpred_t,y_t) - \inf_{f \in \mathcal{F}^{\delta}} \sum_{t=1}^n \logloss(f(x_t),y_t)\right]\\
                    & \leq{} \dtri*{\sup_{x_t\in\cX} \inf_{\lpred_t\in\brk*{\delta,1-\delta}} \max_{y_t \in \crl*{0,1}}}_{t=1}^n\left[ \sum_{t=1}^n \logloss(\lpred_t,y_t) - \inf_{f \in \mathcal{F}^{\delta}} \sum_{t=1}^n \logloss(f(x_t),y_t)\right].
\end{align*}
Since compactness holds, we can apply the minimax theorem and manipulate terms in the same fashion as in the proof of \pref{thm:logistic_minimax} to arrive at the following expression
\begin{align}
&= \sup_{\x, \p} \mathbb{E}_{y \sim \p}\left[ \sum_{t=1}^n \inf_{\lpred_t\in\brk*{\delta,1-\delta}} \mathbb{E}_{y_t \sim \p_t(y)}\left[\logloss(\lpred_t,y_t)\right] - \inf_{f \in \mathcal{F}} \sum_{t=1}^n \logloss(f(\x_t(y)),y_t)\right]\label{eq:logloss_tree}.
\end{align}
In the final line above $\x$ and $\p$ are $\X$- and $\Delta_{\crl*{0,1}}$-valued binary trees (indexed by $\zo$) of depth $n$. The notation ``$y \sim \p$'' refers to the process in which we first draw $y_1 \sim \p_1$, then draw $y_t \sim \p_t(y_1,\ldots,y_{t-1})$ for subsequent timesteps $t$.

Let $V^\gamma$ be a minimal $\gamma$-sequential cover of $\mathcal{F}$ on the tree $\mathbf{x}$ with respect to the $L_{\infty}$ norm in the sense of \pref{def:cover_real}. 

Following the proof of \pref{thm:logistic_minimax}, we define a collection of $\brk*{\delta,1-\delta}$-valued binary trees for each element of $V^{\gamma}$, with the tree $\x$ fixed. For each tree $\vv \in V^\gamma$, each $f \in \F^{\delta}$, we define a class of trees $\cF_{\vv}^{\delta}$ as follows:

\begin{itemize}[leftmargin=*]
\item Initially $\cF^{\delta}_{\vv} = \emptyset$.
\item For each $f\in\Fclip$ and $y\in\crl*{0,1}^{n}$ with $\max_{t\in\brk*{n}}\abs*{f(\x_{t}(y))-\vv_{t}(y)} \le \gamma$: 
\begin{itemize}
\item Define a $\brk*{\delta, 1-\delta}$-valued binary tree $\uu_{f,y}$ via: For each $y'\in\pmo^{n}$, 
\[
  (\uu_{f,y})_{t}(y'_{1:t-1}) \ldef f(\x_{t}(y'))\ind\crl*{y'_1=y_1,\ldots,y'_{t-1}=y_{t-1}} + \vv_{t}(y')\ind\crl*{\neg{}(y'_1=y_1,\ldots,y'_{t-1}=y_{t-1})}.
\]
(So that $\uu_{f,y}$ is equal to $f\circ\x$ on the path $y$, and equal to $\vv$ everywhere else.)
\item Add $\uu_{f,y}$ to $\cF_{\vv}^{\delta}$.
\end{itemize}
\end{itemize}

Just like the construction in \pref{thm:logistic_minimax}, $\cF_{\vv}^{\delta}$ has two properties: Its $L_{\infty}$ covering number is bounded in terms of the $L_{\infty}$ covering number of the class $\Fclip\circ{}\x$, and its $L_{\infty}$ radius is bounded by $\gamma$. These properties are stated in \pref{lem:fv_properties_logloss}.

On any path $y \in \zo^n$ and for each $f \in \F$, there exist $\vv \in V^\gamma$ and $\uu \in \F_\vv^{\delta}$ such that $f(\x_t(y)) = \uu_{t}(y)$. This is because a $\vv$ that is $\gamma$-close to $f$ on the path $y$ through $\x$ is guaranteed by the cover property of $V^{\gamma}$, and so we can take $\uu_{f,y}$ in $\cF_{\vv}$ as the desired $\uu$. This implies that
$$
\inf_{f \in \F }\sum_{t=1}^n \logloss(f(\x_t(y)),y_t) \ge \min_{\vv \in V^\gamma} \inf_{\uu \in \F_\vv^{\delta} }\sum_{t=1}^n \logloss(\uu_t(y),y_t).
$$
Returning to the minimax rate, all the properties of the tree families we have established so far imply
\begin{align}
\notag & \mathcal{V}^{\mathrm{log}}_n(\F^{\delta})  \\ & \le \sup_{\x}\inf_{V^{\gamma}}\inf_{\crl*{\cF_{\vv}^{\delta}}_{\vv\in{}V^{\gamma}}}\sup_{\p}\mathbb{E}_{y \sim \p}\left[ \sum_{t=1}^n \inf_{\lpred_t\in\brk*{\delta, 1-\delta}} \mathbb{E}_{y_t \sim \p_t(y)}\left[\logloss(\lpred_t,y_t)\right] - \min_{\vv \in V^\gamma} \inf_{\uu \in \mathcal{F}_\vv^{\delta}} \sum_{t=1}^n \logloss(\uu_t(y),y_t)\right]. \notag
\end{align}
As in the proof of \pref{thm:logistic_minimax}, we introduce a family of trees representing the minimax optimal strategy competing with each tree class $\cF_{\vv}^{\delta}$. For each $\vv\in{}V^{\gamma}$, we introduce a $\brk*{\delta, 1-\delta}$-valued binary tree $\phtree^{\vv}$, with $\phtree^{\vv}_{t}:\zo^{t-1}\to\deltarange$.
{\small
\begin{align}
  & \begin{aligned}= \sup_{\x}\inf_{V^{\gamma}}\inf_{\crl*{\cF_{\vv}^{\delta}}_{\vv\in{}V^{\gamma}}}\inf_{\crl*{\phtree^{\vv}}_{\vv\in{}V^{\gamma}}}\sup_{\p} \mathbb{E}_{y \sim \p}\biggl[& \sum_{t=1}^n \inf_{\hat{p}_t\in\deltarange} \mathbb{E}_{y_t \sim \p_t(y)}\left[\logloss(\hat{p}_t,y_t)\right] \\
    &\hspace{-3em}- \min_{\vv \in V^\gamma}\left\{ \sum_{t=1}^n \logloss(\phtree_t^\vv(y),y_t) - \sum_{t=1}^n \logloss(\phtree_t^\vv(y),y_t) +  \inf_{\uu \in \mathcal{F}_\vv^{\delta}} \sum_{t=1}^n \logloss(\uu_t(y),y_t)\right\}\biggr].\end{aligned} \notag \\
& \leq\sup_{\x}\inf_{V^{\gamma}}\inf_{\crl*{\cF_{\vv}^{\delta}}_{\vv\in{}V^{\gamma}}}\inf_{\crl*{\phtree^{\vv}}_{\vv\in{}V^{\gamma}}}\left\{\begin{aligned}~& \underbrace{\sup_{\p} \mathbb{E}_{y \sim \p}\biggl[ \sum_{t=1}^n \inf_{\hat{p}_t\in\deltarange} \mathbb{E}_{y_t \sim \p_t(y)}\left[\logloss(\hat{p}_t,y_t)\right] - \min_{\vv \in V^\gamma}\sum_{t=1}^n \logloss(\phtree_t^\vv(y),y_t)\biggr]}_{(\star)} \\
  & + \underbrace{\sup_{\p}\mathbb{E}_{y \sim \p}\left[ \max_{\vv \in V^\gamma}\left\{ \sum_{t=1}^n \logloss(\phtree_t^\vv(y),y_t) -  \inf_{\uu \in \mathcal{F}_\vv^{\delta}} \sum_{t=1}^n \logloss(\uu_t(y),y_t)\right\}\right]}_{(\star\star)}
  \end{aligned}\right\}.\label{eq:star_logloss}
\end{align}}

We now bound the terms $(\star)$ and $(\star\star)$ individually as follows:

\paragraph{Term $(\star)$}
We select $(\hat{p}_t)_{t\leq{}n}$ using the Aggregating Algorithm as configured in \pref{lem:vovk_logloss}, with $W$ as the finite collection of sequences $\crl*{\phtree^{\vv}}_{\vv\in{}V^{\gamma}}$. This is possible because $\phtree^{\vv}_{t}$ only depends on $y_{1:t-1}$.
\[
\sum_{t=1}^n \logloss(\hat{p}_t,y_t) - \min_{\vv \in V^\gamma}\sum_{t=1}^n \logloss(\phtree_t^\vv(y),y_t) \leq{} \log\abs*{V^{\gamma}} + 2.
\]
Since the algorithm's predictions lie in $\deltarange$ they are a valid choice for the infimum in $(\star)$. This implies that
\[
(\star) \leq{} \sup_{\x}\log\abs*{V^{\gamma}}\leq \log\cN_{\infty}(\gamma, \cF^{\delta}).
\]

\paragraph{Term $(\star\star)$}

First, note that we can take each tree class $\cF_{\vv}^{\delta}$ to be $\deltarange$-valued without loss of generality. We exhibit a deterministic strategy for each class by invoking the generic minimax regret bound \pref{lem:logloss_covering}. Since the collection is $\deltarange$-valued, the lemma guarantees existence of a deterministic strategy $(\hat{p}_t)_{t\leq{}n}$ with a regret bound of
{\small
  \begin{align*}
    &\sum_{t=1}^{n}\logloss(\hat{p}_t, y_t) - \inf_{\uu\in{}\cF_{\vv}^{\delta}}\sum_{t=1}^{n}\logloss(\uu_{t}(y), y_t) \\
    &\leq{}
    2n\delta\log(1/\delta) \\
    &~~~~+ \frac{C}{\delta}\log\cN_{\infty}(\gamma, \cF_{\vv}^{\delta}) + \inf_{\alpha\in(0, \gamma]}\crl*{
      \frac{4n\alpha}{\delta} + 30\sqrt{\frac{2n}{\delta}}\int_{\alpha}^{\gamma}\sqrt{\log\cN_{\infty}(\rho,\cF_{\vv}^{\delta})}d\rho
      + \frac{8}{\delta}\int_{\alpha}^{\gamma}\log\cN_{\infty}(\rho,\cF_{\vv}^{\delta})d\rho
      }.
  \end{align*}}

  By \pref{lem:fv_properties_logloss}, $\mathcal{N}_{\infty}(\gamma, \cF_{\vv}^{\delta} ) \leq{} 1$, and so we can drop the leading covering number term in the bound:
  {\small
  \[
    \leq{}
    2n\delta\log(1/\delta)
    + \inf_{\alpha\in(0, \gamma]}\crl*{
      \frac{4n\alpha}{\delta} + 30\sqrt{\frac{2n}{\delta}}\int_{\alpha}^{\gamma}\sqrt{\log\cN_{\infty}(\rho,\cF_{\vv}^{\delta})}d\rho
      + \frac{8}{\delta}\int_{\alpha}^{\gamma}\log\cN_{\infty}(\rho,\cF_{\vv}^{\delta})d\rho
    }.
  \]}
  \pref{lem:fv_properties_logloss} also implies that we can upper bound the covering number in terms of that of $\cF^{\delta}$:
    {\small\[
    \leq{}
    2n\delta\log(1/\delta)
    + \inf_{\alpha\in(0, \gamma]}\crl*{
      \frac{4n\alpha}{\delta} + 30\sqrt{\frac{2n}{\delta}}\int_{\alpha}^{\gamma}\sqrt{\log(n\cN_{\infty}(\rho,\cF^{\delta}, \x))}d\rho
      + \frac{8}{\delta}\int_{\alpha}^{\gamma}\log(n\cN_{\infty}(\rho,\cF^{\delta}, \x))d\rho
    }.
  \]}
  Since the regret inequality holds deterministically and uniformly for all sequences $y$ for each $\vv$, we have that
{\small\begin{align*}
   &(\star\star) \\&    \leq{}
    2n\delta\log(1/\delta)
    + \inf_{\alpha\in(0, \gamma]}\crl*{
      \frac{4n\alpha}{\delta} + 30\sqrt{\frac{2n}{\delta}}\int_{\alpha}^{\gamma}\sqrt{\log(n\cN_{\infty}(\rho,\cF^{\delta}, \x))}d\rho
      + \frac{8}{\delta}\int_{\alpha}^{\gamma}\log(n\cN_{\infty}(\rho,\cF^{\delta}, \x))d\rho
    }.
 \end{align*}}

  \paragraph{Final bound}
We combine $(\star)$ and $(\star\star)$, take the supremum over $\x$, and apply \pref{prop:clip_covering} to conclude that $\mathcal{V}^{\mathrm{log}}_n(\F)$ is bounded by
{\small
\begin{align*}
  &3n\delta\log(1/\delta)
  + \log\cN_{\infty}(\gamma, \cF)\\
    &~~~~+ \inf_{\alpha\in(0, \gamma]}\crl*{
      \frac{4n\alpha}{\delta}
       + 30\sqrt{\frac{2n}{\delta}}\int_{\alpha}^{\gamma}\sqrt{\log(n\cN_{\infty}(\rho,\cF))}d\rho
      + \frac{8}{\delta}\int_{\alpha}^{\gamma}\log(n\cN_{\infty}(\rho,\F))d\rho
    }.
\end{align*}}
The theorem statement uses that we are free to choose any value for $\delta$ and $\gamma$.
\end{proof}

The remaining lemmas in this section mirror those used in the proof of \pref{thm:logistic_minimax}, with the most substantive difference being that we required a more refined chaining bound for general classes under the log loss from \cite{RakSri15}. We omit their proofs.

\begin{lemma}
\label{lem:fv_properties_logloss}
Let $\cF_{\vv}^{\delta}$ be defined as in the proof of \pref{thm:logistic_minimax} for trees $\vv$ and $\x$ and scale $\gamma$. Then it holds that
\begin{enumerate}
\item $\mathcal{N}_{\infty}(\gamma, \cF_{\vv}^{\delta} ) \leq{} 1$.
\item $\mathcal{N}_{\infty}(\alpha, \cF^{\delta}_{\vv} ) \leq{} n\cdot{}\mathcal{N}_{\infty}(\alpha, \cF^{\delta}, \x )$ for all $\alpha>0$.
\end{enumerate}
\end{lemma}
Note that the covering number (\pref{def:cover_real_tree}) was defined for trees indexed by $\pmo^{n}$, but trees in $\cF_{\vv}^{\delta}$ are indexed by $\zo^{n}$. We overload the covering number in the natural way in the lemma above and subsequent lemmas.

\begin{lemma}[\cite{PLG}]
  \label{lem:vovk_logloss}
  Let $W$ be any class of $\deltarange$-valued binary trees of depth $n$. Then Vovk's Aggregating Algorithm configured with $W$ as a benchmark class of experts generates predictions $(\hat{p}_t)_{t\leq{}n}$ that enjoy regret
  \begin{equation}
    \label{eq:logistic_finite}
    \sum_{t=1}^{n}\logloss(\hat{p}_t, y_t) - \min_{\ww\in{}W}\sum_{t=1}^{n}\logloss(\ww_{t}(y), y_t) \leq{} \log\abs*{\cW}.
  \end{equation}
  Furthermore, the predictions $(\hat{p}_t)_{t\leq{}n}$ lie in $\deltarange$.
\end{lemma}

\begin{lemma}[Extracted from \cite{RakSri15}]
  \label{lem:logloss_covering}
  Let $W$ be any class of $\deltarange$-valued binary trees of depth $n$. Then there exists a deterministic prediction strategy $(\hat{p}_t)_{t\leq{}n}$ that enjoys regret
  \begin{align*}
    &\sum_{t=1}^{n}\logloss(\hat{p}_t, y_t) - \inf_{\ww\in{}W}\sum_{t=1}^{n}\logloss(\ww_{t}(y), y_t) \\
    &\leq{}
    2n\delta\log(1/\delta)
    + \frac{C}{\delta}\log\cN_{\infty}(\gamma, W) \\ 
    &~~~~+ \inf_{\alpha\in(0, \gamma]}\crl*{
      \frac{4n\alpha}{\delta} + 30\sqrt{\frac{2n}{\delta}}\int_{\alpha}^{\gamma}\sqrt{\log\cN_{\infty}(\rho,W)}d\rho
      + \frac{8}{\delta}\int_{\alpha}^{\gamma}\log\cN_{\infty}(\rho,W)d\rho
      },
  \end{align*}
  for all $\gamma>0$ and for some absolute constant $C>0$. The predictions $(\hat{p}_t)_{t\leq{}n}$ lie in $\deltarange$.
\end{lemma}


\section{Efficient Implementation}
\label{app:efficient}

In this section we discuss an efficient (i.e. polynomial time in the parameters of the problem) randomized implementation of \pref{alg:mixing_multiclass}. The main idea is to exploit the log-concavity of the density of $P_t$ in the algorithm and to use well-established Markov chain Monte Carlo samplers for such densities to collect enough matrices $W$ sampled from the distribution to approximate the prediction $\zh_t$ sufficiently well to ensure the increase in regret is small.

Fix a round $t$. Recall that the density on $\cW$ we wish to sample from in round $t$ of the algorithm is 
\[
  P_{t}(W) \propto \exp\prn{-\tfrac{1}{L}\textstyle{\sum}_{s=1}^{t-1}\ls(Wx_s, y_s)}.
\]
For notational convenience, define the function $F_t: \cW \rightarrow \R$ as $F_t(W) := \exp\prn{-\tfrac{1}{L}\textstyle{\sum}_{s=1}^{t-1}\ls(Wx_s, y_s)}$. It is easy to check that $F_t$ is log-concave.
\begin{assumption}
\label{ass:sampler}
We have access to a sampler that makes $\text{poly}(1/\veps, n, d, B, R)$ queries to $F_t$ and produces a sample $W$ with distribution $\tilde{P}_t$ such that $d_{\textrm{TV}}(\tilde{P}_t, P_t)\leq{}\veps$.
\end{assumption}

Such samplers are well-known in the literature: for example, the hit-and-run sampler \citep{lovasz2006fast}, the projected Langevin Monte Carlo sampler \citep{bubeck2015sampling}, and the Dikin walk sampler \citep{narayanan2017efficient}. It is easy to derive appropriate bounds on all the relevant parameters of $F_t$ that are involved in the analysis of these samplers so that the samplers run in polynomial time. While this gives a theoretically efficient implementation, the running time bounds are too loose to be practical (for example, see the calculations below for projected Langevin Monte Carlo sampler). We have not attempted to improve these running time bounds; that is a direction for future work.

\begin{example}[\cite{bubeck2015sampling}]
Let $W$ have density $P\propto e^{-f}$ for some $\beta$-smooth, $S$-Lipschitz convex function $f$ over a convex body $\cW$ contained in a euclidian ball of radius $D$ in dimension $d$.
Projected Langevin Monte Carlo produces a sample from $\wt{P}$ with $d_{\textrm{TV}}(\wt{P}, P)\leq{}\veps$ after
$O\prn*{
\frac{D^{6}\max\crl*{d, DS, D\beta}^{12}}{\veps^{12}}
}
$
evaluations. For our setting, when $\nrm*{x_t}_{2}\leq{}R$ and $\nrm*{y_t}_{1}\leq{}L$, the loss $w\mapsto{}\ls(\tri*{w,x_t}, y_t)$ is $O(RL)$-Lipschitz and smooth. We therefore have $S,\beta\leq{}RLn$ and $D=B$, which yields the following bound on the number of queries to $F_t$:
\[
O\prn*{
\frac{B^{6}\max\crl*{dK, BRLn}^{12}}{\veps^{12}}
}.
\]

\end{example}

Given access to a sampler, we can now prove \pref{prop:alg_polytime}. In the following, we use the phrase ``with high probability'' to indicate that the statement referred to holds with probability at least $1 - \delta$ for any $\delta > 0$. We also use the notation $\tilde{O}(\cdot)$ and $\tilde{\Omega}(\cdot)$ to suppress logarithmic dependence on $1/\delta$, $d$, $K$, and $n$.
\begin{proof}[Proof of \pref{prop:alg_polytime}.]
The idea is very straightforward: for some parameters $m \in \mathbb{N}$ and $\veps > 0$ to be specified later, in each round $t$, simply use the sampler with error tolerance $\frac{\veps}{2}$ repeatedly $m$ times to collect samples $W^{(i)}$ for $i \in [m]$ and then approximate the prediction by $\zt_t = \mb{\sigma}^{+}\prn*{\smooth\prn*{\En_{i \sim [m]}\brk*{\mb{\sigma}(W^{(i)}x_t)}}}$. Here, ``$i \sim [m]$'' denotes sampling $i$ uniformly from $[m]$, and $m = \text{poly}(n, d, B, R, 1/\delta)$ will be chosen to be large enough to ensure that this approximation incurs only $1/n$ additional loss in each round, with high probability, and thus at most $O(1)$ additional loss over all $n$ rounds.

\newcommand{\tp}{\tilde{p}}
\newcommand{\ttp}{\tilde{\tilde{p}}}

It remains to provide appropriate bounds on $m$. In the following, we will fix the round $t$ and drop the subscript $t$ from $P_t, \tilde{P}_t, x_t, y_t$, etc. for notational clarity. 

Define the distributions $p =\smooth\prn*{\En_{W \sim P}\brk*{\mb{\sigma}(Wx)}}$, $\tp = \smooth\prn*{\En_{W \sim \tilde{P}}\brk*{\mb{\sigma}(Wx)}}$ and $\ttp = \smooth\prn*{\En_{i \sim [m]}\brk*{\mb{\sigma}(W^{(i)}x)}}$. Then standard Chernoff-Hoeffding bounds and a union bound over all $k \in [K]$ imply that if $m = \tilde{\Omega}\prn*{1/\veps^2}$, then with high probability, we have $\|\tp - \ttp\|_\infty \leq \frac{\veps}{2}$. Furthermore, the sampler ensures $d_{\textrm{TV}}(\tilde{P}, P)\leq{} \frac{\veps}{2}$, which implies that $\|p - \tp\|_\infty \leq \frac{\veps}{2}$ since each coordinate of $p$ and $\tp$ are i n $[0, 1$. Thus, by the triangle inequality, we have $\|p - \ttp\|_\infty \leq \veps$.

We now bound the excess loss for using $\ttp$ instead of $p$ in the algorithm, using the fact the weighted multiclass logistic loss can be equivalently viewed as a weighted multiclass log loss after passing the logits through the softmax function $\mb{\sigma}$. Thus, the additional loss equals
\[ \sum_{k \in [K]} y_k \log\prn*{\tfrac{p_k}{\ttp_k}} \leq \sum_{k \in [K]} y_k \log\prn*{\tfrac{\ttp_k + \veps}{\ttp_k}} \leq \sum_{k \in [K]} y_k \log\prn*{1 + \tfrac{\veps K}{\mu}} \leq \tfrac{\veps KL}{\mu}.\]
The first inequality above follows from the bound $\|p - \ttp\|_\infty \leq \veps$, and the second from the fact that $\ttp_k \geq \frac{\mu}{K}$ for all $k \in [K]$, and the third from $\log(1 + a) \leq a$ for all $a \in \R_+$ and $\|y\|_1 \leq L$. Thus, setting $\veps = \frac{\mu}{KLn}$ ensures that the additional loss is at most $1/n$ with high probability, as required.
\end{proof}


\end{document}